\documentclass{article}

\usepackage[preprint]{neurips_2025}

\usepackage[utf8]{inputenc} 
\usepackage[T1]{fontenc}    
\usepackage{hyperref}       
\usepackage{url}            
\usepackage{booktabs}       
\usepackage{amsfonts}       
\usepackage{subfig}         
\usepackage{nicefrac}       
\usepackage{microtype}      
\usepackage{xcolor}         
\usepackage{amsmath}
\usepackage{amssymb}
\usepackage{mathtools}
\usepackage{amsthm}

\title{Neural Multivariate Regression: Qualitative Insights from the Unconstrained Feature Model}

\author{%
  George Andriopoulos$^{1}$\thanks{Corresponding author: \texttt{ga73@nyu.edu}} \quad Soyuj Jung Basnet$^{3}$ \quad Juan Guevara $^{4}$ \quad Li Guo$^{2}$ \quad Keith Ross$^{1}$\AND
  \text{\normalfont $^1$ New York University Abu Dhabi \quad
$^2$ New York University Shanghai}\\ $^3$ New York University \quad $^4$ MBZUAI
}

\theoremstyle{plain}
\newtheorem{theorem}{Theorem}[section]

\newtheorem{corollary}[theorem]{Corollary}
\theoremstyle{definition}
\newtheorem{definition}[theorem]{Definition}

\theoremstyle{remark}
\newtheorem{remark}[theorem]{Remark}

\newcommand{\bP}{{\mathbf P}}
\newcommand{\bW}{{\mathbf W}}
\newcommand{\bY}{{\mathbf Y}}
\newcommand{\bR}{{\mathbf R}}
\newcommand{\bA}{{\mathbf A}}
\newcommand{\bH}{{\mathbf H}}
\newcommand{\bI}{{\mathbf I}}
\newcommand{\bZ}{{\mathbf Z}}
\newcommand{\bC}{{\mathbf C}}
\newcommand{\bV}{{\mathbf V}}
\newcommand{\bh}{{\mathbf h}}

\newcommand{\bzero}{{\mathbf 0}}

\newcommand{\by}{{\mathbf y}}
\newcommand{\byb}{{\bar{\mathbf y}}}

\newcommand{\bYb}{{\bar{\mathbf Y}}}
\newcommand{\bx}{{\mathbf x}}
\newcommand{\bb}{{\mathbf b}}

\newcommand{\bc}{{\mathbf c}}
\newcommand{\bw}{{\mathbf w}}

\newcommand{\bSigma}{{\mathbf \Sigma}}
\newcommand{\byi}{{\mathbf y}_i}
\newcommand{\bxi}{{\mathbf x}_i}

\newcommand{\lH}{\lambda_{\bH}}
\newcommand{\lW}{\lambda_{\bW}}
\newcommand{\lw}{\lambda_{\bw}}

\newcommand{\Rd}{\mathbb{R}^d}

\renewcommand{\eqref}[1]{(\ref{#1})}

\begin{document}

\maketitle

\begin{abstract}
The Unconstrained Feature Model (UFM) is a mathematical framework that enables closed-form approximations for minimal training loss and related performance measures in deep neural networks (DNNs). This paper leverages the UFM to provide qualitative insights into neural multivariate regression, a critical task in imitation learning, robotics, and reinforcement learning. Specifically, we address two key questions: (1) How do multi-task models compare to multiple single-task models in terms of training performance? (2) Can whitening and normalizing regression targets improve training performance? The UFM theory predicts that multi-task models achieve strictly smaller training MSE than multiple single-task models when the same or stronger regularization is applied to the latter, and our empirical results confirm these findings. Regarding whitening and normalizing regression targets, the UFM theory predicts that they reduce training MSE when the average variance across the target dimensions is less than one, and our empirical results once again confirm these findings.  These findings highlight the UFM as a powerful framework for deriving actionable insights into DNN design and data pre-processing strategies.
\end{abstract}

\section{Introduction}

Deep neural networks (DNNs) have become a cornerstone of modern machine learning, enabling transformative advancements in fields such as computer vision, natural language processing, and robotics. These models, often comprising millions to trillions of parameters, are trained by minimizing regularized loss functions over a high-dimensional parameter space. However, the non-convexity and high dimensionality of these loss functions make it practically impossible to derive closed-form expressions for their minimum value or for performance metrics, such as cross-entropy loss or mean-squared error (MSE), evaluated at a global minimum. This limitation has impeded theoretical exploration and limited our understanding of how design choices influence DNN behavior.

A recent breakthrough in this domain is the Unconstrained Feature Model (UFM) and the closely related Layer-Peeled Model \citep{fang2021exploring, mixon2020neural}. A typical DNN consists of a nonlinear feature extractor followed by a final linear layer. Inspired by the universal approximation theorem \citep{cybenko1989approximation, hornik1989multilayer}, the UFM assumes that the feature extractor can map any set of training examples to any desired set of feature vectors. The UFM framework further simplifies the problem by replacing L2-regularization of the feature extractor parameters with  L2-regularization of the feature vectors themselves. These assumptions lead to a new optimization problem that, while still non-convex, is mathematically tractable. In particular, the UFM allows for the derivation of closed-form expressions for minimal training loss, providing a powerful theoretical tool to analyze DNN behavior.

In this paper, we explore whether the closed-form expressions derived from the UFM can yield qualitative insights into the behavior of DNNs. We focus specifically on neural multivariate regression, where the DNN predicts a vector of targets. This task is central to several important applications, including imitation learning in autonomous driving and robotics, as well as deep reinforcement learning.

For multivariate regression with MSE loss, the UFM decomposes the training MSE into two distinct components: a term linear in the regularization constant; a term that depends on the eigenvalues of the sample covariance matrix of the target data. This decomposition highlights how both regularization and data structure influence training performance. 

Building on these basic UFM results, we investigate two fundamental questions in multivariate regression:
\begin{enumerate}
    \item {\bf Single multi-task model versus multiple single-task models:} Multivariate regression can be approached by training a single multi-task model that predicts all targets simultaneously or by training multiple single-task models, each dedicated to a specific target type. Multi-task models are significantly more efficient in terms of computation and memory. Under the UFM assumption, can we mathematically show that one approach has superior performance over the other?
    \item {\bf Whitening and normalizing regression targets:} Whitening transforms data to have a covariance equal to the identity matrix, while normalization adjusts variances to be equal to one while preserving correlations. These techniques have been extensively applied to inputs and intermediate features but rarely to regression targets. Under the UFM assumption, can we mathematically show that whitening or normalizing the targets improves performance?  
\end{enumerate}
	
To respond to these questions, we combine theoretical insights from the UFM with extensive empirical analyses. Our empirical evaluations are conducted on four datasets: three robotic locomotion datasets and one autonomous driving dataset. Specifically:
\begin{itemize}
    \item For the multi-task model versus multiple single-task models problem, we use UFM theory to show that multi-task models achieve strictly smaller training MSE than the single-task models when using  stronger regularization for the latter. Even with equivalent regularization, multi-task models are theoretically guaranteed to have training MSE less than or equal to that of single-task models. These qualitative results are then confirmed empirically: multi-task models consistently outperform single-task models across all regularization settings in terms of training MSE. 
    \item UFM theory and empirical evidence demonstrate that whitening and normalizing the targets have an important effect on performance. When applied, these techniques adjust the training MSE in a manner dependent on the inherent variance of the target data. Notably, they improve training performance when the average variance per target dimension is less than one, but harm training performance when the converse holds. This highlights the importance of examining the variance structure of target data when considering the pre-processing strategies of whitening and normalization. 
\end{itemize}

Beyond these specific results, our work demonstrates the broader potential of the UFM as a theoretical and practical tool. By enabling closed-form solutions, the UFM provides an opportunity to gain important qualitative insights into DNNs.  Our findings highlight how the UFM can guide design decisions such as choosing between multi-task and single-task models, and on choosing data pre-processing strategies such as whitening and normalization.

\section{Related work}

The development of deep learning has primarily been driven by heuristics and empirical experience, with limited progress in establishing a solid theoretical framework. The primary challenge lies in the highly non-convex nature of the loss landscape, which complicates optimization and theoretical analysis. Among existing approaches, Neural Tangent Kernel (NTK) theory \citep{arora2019fine, du2018gradient, du2019gradient, jacot2018neural, zou2020gradient} provides a valuable tool for understanding optimization and convergence behavior in neural networks within the infinite-width regime. However, NTK focuses on the early stages of training and neglects the rich nonlinear feature learning that characterizes practical neural networks. 

In contrast, the Unconstrained Feature Model (UFM) \citep{mixon2020neural} and the Layer-Peeled Model \citep{fang2021exploring} offer an alternative perspective that emphasizes the nonlinearity of feature representations. 
Unlike NTK, which focuses on parameter dynamics, UFM assumes universal feature representations and treats the last-layer features as free parameters. This distinctive formulation allows UFM to capture phenomena that NTK cannot explain, such as neural collapse \citep{papyan2020prevalence, weinan2022emergence, zhu2021geometric}, where class features converge to their means, and class means form a simplex equiangular tight frame (ETF).

Research with the UFM includes training on imbalanced data \citep{hong2023neural, thrampoulidis2022imbalance, yanneural, yang2022inducing}, using normalized features \citep{yaras2022neural}, and using various loss functions such as mean-squared error (MSE) \citep{han2021neural, zhou2022optimization}, label smoothing, and focal losses \citep{guo2024cross, zhou2022all}. Additionally, UFM has been applied to analyze multi-label classification \citep{li2023neural}, scenarios involving a large number of classes \citep{jiang2023generalized}, and multivariate regression tasks \citep{andriopoulos2024prevalence}. By examining the landscape of the loss function, UFM provides valuable insights into the optimization and convergence behavior of neural networks \citep{yaras2022neural, zhou2022optimization, zhu2021geometric}. 
Recent research has extended the classical UFM framework to investigate the propagation of neural collapse beyond the last layer to earlier layers of DNNs, i.e., UFM with two layers connected by nonlinearity in \citep{tirer2022extended}, UFM with multiple linear layers in \citep{dang2023neural}, deep UFM (DUFM) with nonlinear activations in the context of classification \citep{sukenik2024deep, sukenik2024neural}. Such extensions aim to generalize UFM principles to more complex architectures, thereby bridging the gap between theoretical analysis and real-world DNN implementations.  

There is a growing body of literature that explores trade-offs between multi-task learning (MTL) and multiple single-task learning (STL) for computer vision to train models that can perform multiple vision tasks (e.g., object detection, segmentation, depth estimation) simultaneously \citep{kendall2018multi, misra2016cross, vandenhende2021multi}, and also for reinforcement learning to train a single agent to perform multiple tasks, leveraging shared representations and knowledge transfer across tasks \citep{rusu2015policy, teh2017distral, yu2020meta}. 
Several works explicitly compare MLT and SLT, analyzing their strengths and weaknesses \citep{fifty2021efficiently, ruder2017overview}.

In parallel, normalization and whitening techniques have been widely studied in machine learning. Beyond their conventional application to raw input data \citep{lecun1998gradient}, these techniques have also been applied in intermediate layers to not only normalize \citep{ioffe2015batch}, but to also decorrelate features in order to stabilize training, reduce complexity in latent space representation, and improve convergence \citep{huang2018decorrelated, siarohin2018whitening, vincent2010stacked}. \cite{huang2019iterative} also proposed iterative methods for approximating full whitening in deep layers with the aim of improving optimization efficiency. Furthermore, whitening features helps to reduce domain shifts to align distributions across source and target domains, aiding in transfer learning \cite{sun2016deep}.

Although, there has been significant work on the UFM, 
to the best of our knowledge, this is the first paper that uses the UFM model explicitly in order to mathematically address 
the theoretical behavior of training MSE loss in multi-task regression vs single-tasks regressions, and when using 
target whitening and normalization.  


\section{Approximations motivated by the UFM} \label{ufmapproxsec}

We consider the multivariate regression problem with $M$ training examples $\{(\bxi,\byi), i=1,\ldots,M \}$ with input $\bxi\in \mathbb{R}^D$ and target $\byi\in \mathbb{R}^n$. For univariate regression, $n=1$.
For the regression task, the DNN takes as input an example $\bx \in \mathbb{R}^D$ and produces an output $\by = f(\bx) \in \mathbb{R}^n$. For most DNNs, including those used in this paper, this mapping takes the form
$f_{\theta,\bW,\bb}(\bx)=\bW \bh_{\mathbf{\theta}}(\bx) + \bb$,
where $\bh_{\mathbf{\theta}}(\cdot): \mathbb{R}^D\to \Rd$ is the nonlinear feature extractor consisting of several nonlinear layers, $\bW$ is a $n \times d$ matrix representing the final linear layer in the model, and $\bb \in \mathbb{R}^n$ is the bias vector. 
The parameters ${\mathbf{\theta}}$, $\bW$, and $\bb$ are all trainable.

We typically train the DNN using gradient descent to minimize the regularized L2 loss: 
\begin{align} \label{standard_loss}
   \mathcal{L}_{\mbox{params}}(\theta, \bW,\bb) := 
    \frac{1}{2M} \sum_{i=1}^{M} ||f_{\theta,\bW,\bb}(\bxi) - \byi||_2^2 
    + \frac{\lambda_{\theta}}{2} ||\mathbf{\theta}||_2^2 +  \frac{\lW}{2}||\bW||_F^2,
\end{align}
where $||\cdot||_2$ and $||\cdot||_F$ denote the $L_2$-norm and the Frobenius norm respectively. As commonly done in practice, in our experiments we set all the regularization parameters to the same value, which we refer to as the weight decay parameter $\lambda_{WD}$, that is, we set $\lambda_{\theta} = \lW =\lambda_{WD}$.

In this paper, we consider a modified version  of the standard problem, namely, to train the DNN to minimize
\begin{align} \label{modified_loss}
    \mathcal{L}_{\mbox{features}}(\theta, \bW,\bb) :=
    \frac{1}{2M} \sum_{i=1}^{M} ||f_{\theta,\bW,\bb}(\bxi) - \byi||_2^2
    + \frac{\lH}{2M} \sum_{i=1}^{M} ||\bh_{\theta}(\bx_i) ||_2^2  + \frac{\lW}{2}||\bW||_F^2,
\end{align}
where $\lH$ and $\lW$ are non-negative regularization parameters. 
We refer to the standard loss function (\ref{standard_loss}) as the {\em parameter-regularized loss function} and to the modified loss function (\ref{modified_loss}) as the {\em feature-regularized loss function}. Although the two problems are not the same, by regularizing the features, we are implicitly constraining the internal parameters $\theta$.

Our goal is to derive mathematically motivated approximations for the training error related with  (\ref{modified_loss}) and then use these approximations to gain qualitative insights into fundamental issues in neural multivariate regression. To this end, we consider  the \textit{Unconstrained Feature Model (UFM)-loss function} \citep{fang2021exploring, mixon2020neural} which is defined as follows: 
\begin{align} \label{formofloss}
   \mathcal{L}(\bH, \bW,\bb) 
     := \frac{1}{2M} ||\bW \bH + \bb \mathbf{1}_M^{T} - \bY||_F^2 
     + \frac{\lH}{2M} ||\bH||_{F}^2 
     + \frac{\lW}{2} ||\bW||_F^2,
\end{align}
$\bH := [\bh_1 \cdots \bh_M]\in \mathbb{R}^{d\times M}$, $\bY := [\by_1 \cdots \by_M]\in \mathbb{R}^{n\times M}$. Note that the UFM-loss function solely depends on $\bH$, $\bW$ and $\bb$ and does not depend on the inputs or the parameters $\theta$. It is easily seen that $\min_{\bH, \bW, \bb} \mathcal{L}(\bH, \bW,\bb)\le \min_{\theta, \bW, \bb} \mathcal{L}_{\mbox{features}}(\theta, \bW,\bb)$. The minimal values of the two loss functions are equal if the feature extractor is so expressive that any feature vector configuration is attainable for the inputs in the training set. 
This simple observation is very powerful since minimizing the UFM-loss function is mathematically tractable and can be solved in closed form without requiring any training.

\subsection{Closed-form expressions}

Let  $\bSigma$ denote the $n \times n$ sample covariance matrix corresponding to the targets $\{\byi, i=1,\ldots,M \}$: $\bSigma = M^{-1}(\bY - \bYb) (\bY - \bYb)^{T}$, where $\bYb = [\byb \cdots \byb]$, and $\byb = M^{-1} \sum_{i=1}^M \byi$.  
Throughout this paper, we make the natural assumption that $\bY$ and $\bSigma$ have full rank. Thus $\bSigma$ has a positive definite square root, which we denote by $\bSigma^{1/2}$. We define the \textit{training MSE} for the UFM as:
\begin{equation}
    \text{MSE}(\bH, \bW, \bb):=\frac{1}{M}||\bW \bH + \bb \mathbf{1}_M^T-\bY||_F^2.
\end{equation}
 Reorder the eigenvalues of $\bSigma$ so that 
$\lambda_{\max} := \lambda_1 \geq \lambda_2 \geq \cdots \geq \lambda_n := \lambda_{\min}> 0$.
Let $c := \lH \lW$ and
let $j^{*} := \max \{ j : \lambda_j \geq c \}$ with the convention $j^{*}=0$ when the set in question is the empty set. When $n=1$, let $\sigma^2$ denote the variance of the 1-d targets over the $M$ samples. 
For any $p\times q$ matrix $\bC$ with columns $\bc_1, \bc_2,\ldots,\bc_q$, we denote $[\bC]_j := [\bc_1 \; \bc_2 \cdots\bc_j \; \bzero \cdots \bzero]$. 
Our main analytical results will be based on the closed-form solutions given below.

\begin{theorem} \label{gendim} 
Suppose $(\bH^*,\bW^*,\bb^*)$ minimizes the UFM-loss $\mathcal{L}(\bH, \bW,\bb)$ given by \eqref{formofloss}. Then,
\begin{equation} \label{loss_train}
\mathcal{L}(\bH^*, \bW^*, \bb^*)= \textnormal{MSE}(\bH^*, \bW^*, \bb^*) + \sqrt{c} \sum_{i=1}^{n} \eta_i,
\end{equation}
where $\eta_i$ is the $i$-th diagonal entry of $[\bSigma^{1/2}-\sqrt{c} \bI_n]_{j^{*}}$, 
and
\begin{equation} \label{MSE_train}
\textnormal{MSE}(\bH^*, \bW^*, \bb^*)=j^{*} c + \sum_{i=j^{*}+1}^n \lambda_i.
\end{equation}
If $n=1$, we have that
\begin{equation} \label{uniMSE}
\textnormal{MSE}(\bH^*, \bw^*,  b^*)=
\begin{cases}
c, &\text{ if } c\le \sigma^2,
\\
\sigma^2, &\text{ if } c>\sigma^2.
\end{cases}
\end{equation}
\end{theorem}


\section{Multi-task vs multiple single-task models}
In this section, we leverage the closed-form solutions provided in Theorem \ref{gendim} to study a fundamental problem in multivariate regression: which is better, a single multi-task model with $n$ tasks (target types) or $n$ dedicated single-task models?  Multi-task regression provides several advantages in terms of computation, memory efficiency, training time, and hyper-parameter tuning. Multi-task regression trains a single DNN to predict multiple outputs simultaneously, enabling weight sharing and the reuse of learned features across tasks. With $n$ dedicated single-task models, we need roughly $n$ times the number of parameters, and we must separately train and perform hyper-parameter tuning for each of the $n$ models. Also, dedicated single-task models are less efficient during inference since the single-task approach requires running the input through each of the $n$ models. 

But how does multi-task regression compare to $n$ single-task regressions in terms of the MSE training performance? We will first address this question mathematically through the lens of the UFM framework. We will then supplement the mathematical results with empirical findings.

\subsection{Insights from the UFM framework} \label{ufmfram}
Consider $n$ univariate regression problems with training sets 
$
    \{(\bx_j, \by_j^{(i)}),j=1,...,M\},
$
where $\by_j^{(i)}$ corresponds to the target value in the $i$-th dimension of the $j$-th training example.  For each one of these univariate problems, consider the corresponding single-task UFM, each with regularization parameters $\tilde{\lambda}_{\mathbf{W}}$ and $\tilde{\lambda}_{\mathbf{H}}$. 
Define $\tilde{c}:=\tilde{\lambda}_{\mathbf{H}} \tilde{\lambda}_{\mathbf{W}}$. 
Let $(\bH^{(i)}, \bw^{(i)}, b^{(i)})$ denote an optimal solution for the $i$-th such single-task model, and let $\text{MSE}^{(i)}(\bH^{(i)}, \bw^{(i)}, b^{(i)})$ denote its corresponding MSE.  The total MSE across the $n$ single-task models is then given by
{\footnotesize %
\begin{align}
\text{MSE}(\text{n-single}, \tilde{c})
:=\sum_{i=1}^n \text{MSE}^{(i)}(\bH^{(i)}, \bw^{(i)}, b^{(i)}).
\end{align}
}

Let $\sigma_i^2$ denote the variance of the targets for the $i$-th single task regression problem: $\sigma_i^2:=M^{-1} \sum_{j=1}^M (\by_j^{(i)}-\byb^{(i)})^2$. Re-order the indices so that $\sigma_1^2\ge \sigma_2^2\ge \cdots \ge \sigma_n^2$. 
Define  $k^{*}:=\max \{j: \sigma_j^2\ge \tilde{c}\}$.  The corollary below follows directly from case $n=1$ of Theorem \ref{gendim}.
\begin{corollary} \label{msensingle} 
The total MSE across the $n$ single-task problems is given by
\begin{equation}
\textnormal{MSE}(\text{n-single}, \tilde{c}) = k^{*} \tilde{c}+\sum_{i=k^{*}+1}^n \sigma_i^2.
\end{equation}
\end{corollary}

We now present the main result of this section. 

\begin{theorem} \label{singlevsmulti}
Let $\text{MSE}(\text{multi}, c)$ be a shorthand of \eqref{MSE_train}.

$(i)$ Suppose $\tilde{c} = c$. Then, 
\[
\textnormal{MSE}(\text{multi}, c) \leq \textnormal{MSE}(\text{n-single}, c).
\]
Furthermore, for $\lambda_{\min} < c < \lambda_{\max}$ and $j^*<k^*$, the inequality is strict, while for  $0<c<\lambda_{\min}$ or $c>\lambda_{\max}$  the inequality holds with equality.

$(ii)$ Suppose $c<\tilde{c}$. Then,
\[
\textnormal{MSE}(\text{multi}, c) \leq \textnormal{MSE}(\text{n-single}, \tilde{c}),
\]
Furthermore, for $c< \min\{\tilde{c}, \lambda_{\max} \}$ the inequality is strict, and
if $\lambda_{\max} < c < \tilde{c}$ the inequality holds with equality.

\end{theorem}

Theorem \ref{singlevsmulti} states that if the multi-task regularization constant $c$ is not greater than the single-task regularization constant $\tilde{c}$, then under the UFM approximation, the MSE training error for the multi-task model is always less than that of the $n$ single-task models. Furthermore, the theorem provides refined information for when the inequality is strict or actually an equality. Thus, not only is the multi-task approach more efficient in terms of memory, training and inference computation, it also has lower training MSE in situations of practical interest under the UFM assumption. We briefly note that the case of $c > \tilde{c}$ remains an open problem for future research. 

\subsection{Experimental results: multi-task vs multiple single-task models} \label{sec:4.2}

\textbf{Datasets.}  Our empirical experiments utilize the Swimmer, Reacher, and Hopper datasets, derived from MuJoCo \citep{ brockman2016openai, mujoco, towers_gymnasium_2023}, a physics engine designed for simulating continuous multi-joint robotic control. These datasets have been widely used as benchmarks in deep reinforcement learning research. Each dataset consists of raw robotic states as inputs and corresponding robotic actions as targets. To adapt them for imitation learning, we reduce their size by selecting a subset of episodes.
In addition, we use the CARLA dataset, sourced from the CARLA Simulator—an open-source platform for autonomous driving research. Specifically, we utilize an expert-driven offline dataset \citep{Codevilla2018}, where input images from vehicle-mounted cameras are recorded alongside corresponding expert driving actions as the vehicle navigates a simulated environment. For convenience, we consider a simplified 2D version, which includes only speed and steering angle.
Table \ref{tab:data} summarizes the datasets, including target dimensions and spectral properties. 

\begin{table}[h!]
\caption{ \textbf{Spectral Properties of Datasets}.  $\lambda_{\min}$, $\lambda_{\max}$, and $\bar{\lambda}$ denote the minimum, maximum, and average eigenvalues of the target variable's covariance matrix. $\mathbf{\tilde{\lambda}_{\min}}$ and $\mathbf{\tilde{\lambda}_{\max}}$ represent the minimum and maximum eigenvalues of the target variable's correlation matrix. }
\begin{center}
\begin{small}
\begin{tabular}{cccccccc}
\toprule
\textbf{Dataset} & $n$ & $\mathbf{\lambda_{\min}}$ & $\mathbf{\tilde{\lambda}_{\min}}$ & $\mathbf{\lambda_{\max}}$ & $\mathbf{\tilde{\lambda}_{\max}}$ & $\mathbf{\bar{\lambda}}$ \\
\midrule

Reacher &  2  & 0.010 & 0.991 & 0.012 & 1.009 & 0.011 \\

Swimmer &  2  & 0.276 & 0.756 & 0.466 & 1.244 & 0.371 \\

Hopper &  3  & 0.215 & 0.782 & 0.442 & 1.258 & 0.345 \\

CARLA 2D & 2 & 0.024 & 0.996 & 209.097 & 1.004 & 104.561\\

\bottomrule

\end{tabular}
\end{small}
\end{center}
\label{tab:data}
\end{table}

\textbf{Experimental settings.} 
For the Swimmer, Reacher, and Hopper datasets, we employed a four-layer MLP (with the last layer being the linear layer) as the policy network for the prediction task. Each layer consisted of 256 nodes, aligning with the conventional model architecture in most reinforcement learning research \citep{tarasov2022corl}. 
For CARLA 2D, we employed ResNet18 \citep{he2016deep} as the backbone model. To minimize the impact of extraneous factors, we applied standard pre-processing techniques without data augmentation. \looseness=-1

All experimental results were averaged over random seeds, with variance across seeds represented by error bars. While weight decay is generally set to small values in practice, we tested a range of values from $1\mathrm{e}{-5}$ to $1\mathrm{e}{-1}$ to thoroughly explore its effect on model training MSE. For each weight decay value, the model was trained for a fixed number of epochs, and training MSE was recorded. The number of training epochs was adjusted based on dataset size, with smaller datasets requiring more epochs to reach convergence. The full experimental setup is provided in Appendix \ref{sec:a_exp}.

\begin{figure*}[htbp]
    \centering
    \begin{tabular}{c c c c c}
        & \textbf{Reacher} & \textbf{Swimmer} & \textbf{Hopper} & \textbf{CARLA 2D} \\
        \raisebox{1.5em}{\rotatebox{90}{\small \textbf{Training MSE}}} &
        \includegraphics[width=0.205\linewidth]{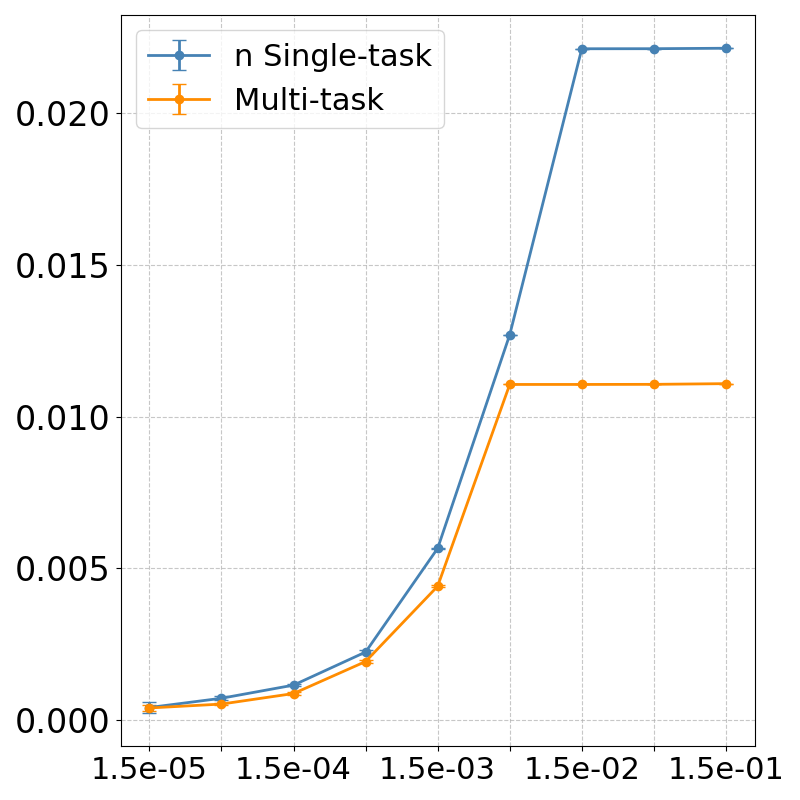} &
        \includegraphics[width=0.205\linewidth]{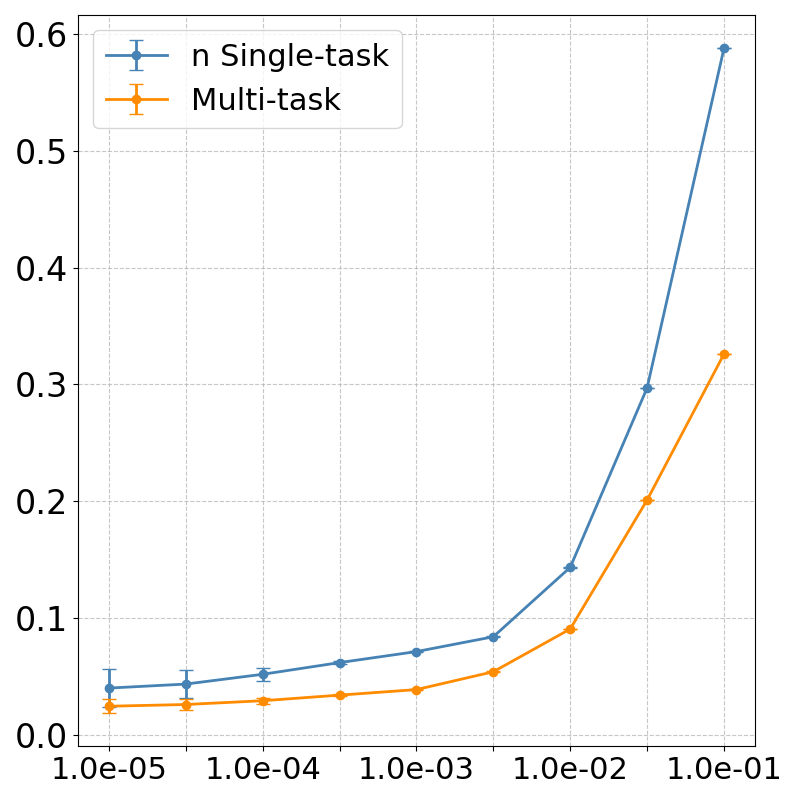} &
        \includegraphics[width=0.205\linewidth]{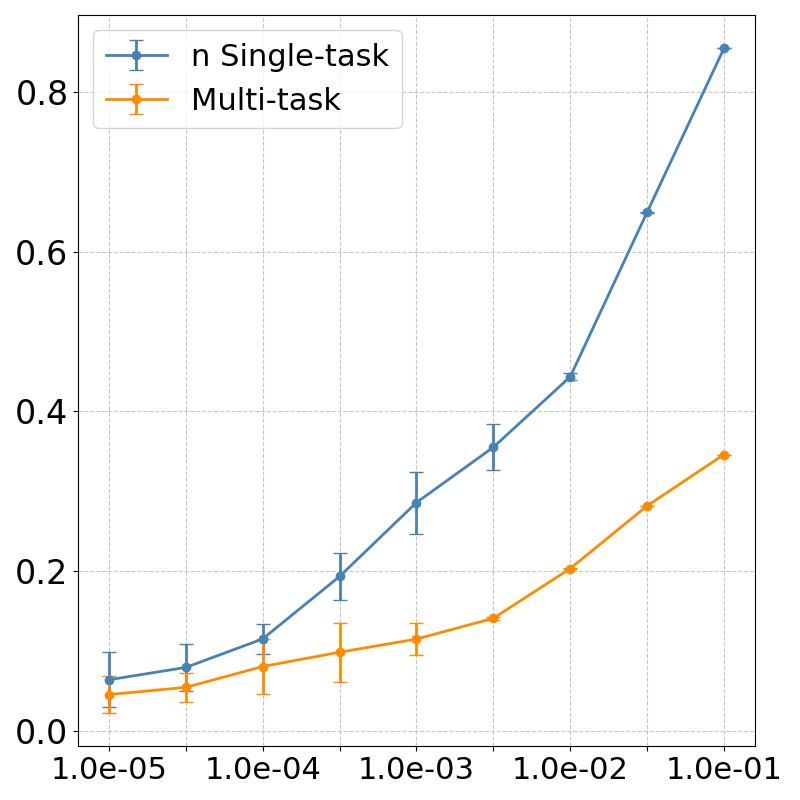} &
        \includegraphics[width=0.205\linewidth]{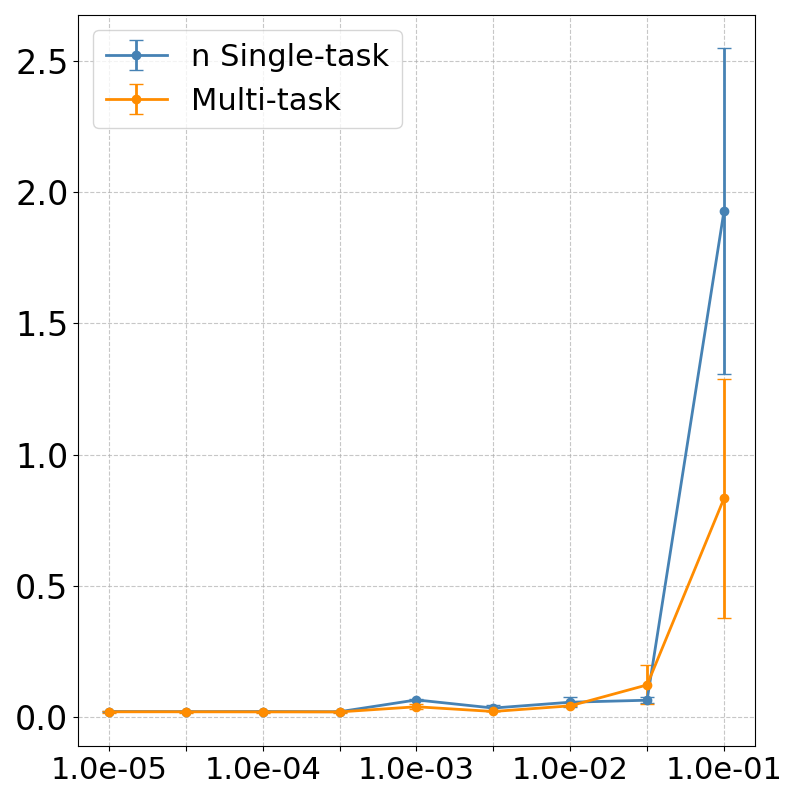}
        \\
        & $\lambda_{WD}$ & $\lambda_{WD}$ & $\lambda_{WD}$ & $\lambda_{WD}$
    \end{tabular}
    \caption{Comparison of the training error of a single multi-task model with that of multiple single task models for different weight decay values after training with the standard parameter-regularized loss function.}
    \label{fig:multi-task-vs-n-single-task-case1}
\end{figure*}

\textbf{Empirical results}. Figure \ref{fig:multi-task-vs-n-single-task-case1} shows, regardless of the value of weight decay, that the multi-task model has lower training MSE than the multiple single-task models. 
These empirical results are consistent with our results in Theorem \ref{singlevsmulti}, 
which established that under the UFM assumption, multi-task models achieve smaller training MSE when using equivalent or stronger regularization for single-task models. These empirical results are invariable to changes in the choice of architecture, cf. Figure \ref{fig:multi-task-vs-n-single-task-case1} with Figures \ref{fig:arch_vs_tasks_reacher}-\ref{fig:arch_vs_tasks_hopper} in Appendix \ref{additexpres}. Moreover, the trends for test MSE align closely with those shown for training MSE in Figure \ref{fig:multi-task-vs-n-single-task-case1}, see Figure \ref{fig:test-mse-multi} in Appendix \ref{additexpres}. We mitigate a theoretical explanation of this alignment in Appendix \ref{genboundappendix}.

We have also performed experiments directly comparing the training MSE predicted by the UFM and the empirical MSE obtained by training the standard parameter-regularized model. We found that the UFM tends to underestimate the empirical MSE values. A challenging direction for future research is to refine the UFM model so that it provides more accurate estimates while remaining mathematically tractable.    

\section{Target whitening and normalization} \label{tarwhitenandnorm}

\textbf{Whitening and decorrelation}. In statistical analysis, whitening (or sphering) refers to a common pre-processing step to transform random variables to orthogonality. A whitening transformation is a linear transformation that converts a random vector with a known covariance matrix into a new random vector of the same dimension and with a covariance matrix given by the identity matrix. Orthogonality among random vectors greatly simplifies multivariate data analysis both from a computational as well as from a statistical standpoint. Whitening is employed mostly in pre-processing but is also part of modeling \citep{hao2015sparsifying, zuber2009gene}. 

In this paper we consider whitening the targets using a natural and common form of whitening called zero-phase component analysis (ZCA). In ZCA, in order to whiten a target vector $\by_i$, we simply subtract from $\by_i$ the mean $\byb$ and then ''divide'' by the square root of the sample covariance matrix. More specifically, 
    \begin{equation}
    \bY^{ZCA}=\bSigma^{-1/2} (\bY-\bYb).
    \end{equation}
\textbf{Connection with UFM}. Importantly, after examining closely the global minima for the UFM-loss function, a significant implication arises. The residual errors will be zero mean and uncorrelated across the $n$ target dimensions, with each variance equal to $c$. More specifically, by \citep{andriopoulos2024prevalence}[Corollary 4.2(v)], the residual error is proportional to the ZCA-whitened targets:
\begin{equation}
\mathbf{E}:=\bW \bH+\bb \mathbf{1}_M^T-\bY=-\sqrt{c} \bY^{ZCA}.
\end{equation}
\textbf{Training MSE with whitened targets}. The emergence of ZCA whitening in the UFM model, as well as the fact that ZCA whitening guarantees no loss of information between the unprocessed and the transformed targets (see the discussion in Appendix \ref{disc}), motivates us to investigate how whitening of the targets will affect the training MSE, from both the UFM and empirical perspectives. For the UFM analysis, we consider the following whitening process: 
\begin{itemize}
    \item First, we whiten the targets using $\bY^{ZCA}=\bSigma^{-1/2} (\bY-\bYb)$.
    \item Next, we obtain the optimal $\tilde{\bW}$ and $\tilde{\bH}$ for the UFM-loss using $\bY^{ZCA}$. 
    \item We then obtain the associated predictions for the whitened training data $\tilde{\bY}:= \tilde{\bW} \tilde{\bH}$.
    \item We then de-whiten these predictions: 
    $\hat{\bY}:=[\bSigma^{1/2}] \tilde{\bY} + \bYb$.
    \item Finally, we calculate the MSE for the whitening approach as 
    \begin{equation}
    \text{MSE}(\text{de-whiten}):=M^{-1} ||\hat{\bY} - \bY||_F^2.
    \end{equation}
\end{itemize}

Following the procedure described in the bullet points above, our next theorem provides a closed-form expression for $\text{MSE(de-whiten)}$. 

\begin{theorem} \label{whitentrain1}
\begin{equation}
\textnormal{MSE}(\textnormal{de-whiten}) = 
\min\{c,1\} \displaystyle \sum_{i=1}^n \lambda_i.
\end{equation}
\end{theorem}

We now examine how $\text{MSE}(\text{de-whiten})$ compares to the training MSE with the unprocessed targets.

\begin{theorem} \label{whitentrain2}
        $(i)$ Suppose $0<c\le 1$. If 
        \begin{equation} \label{newcond}
        \sum_{i=1}^{j^*} \lambda_i - j^* < c^{-1} (1-c)  \sum_{i=j^*+1}^n \lambda_i,
        \end{equation}
        we have that 
        \[
        \text{MSE}(\text{de-whiten})<\text{MSE}(\text{no-whitening}).
        \]
        When the converse inequality of \eqref{newcond} holds, then 
        \[
        \text{MSE}(\text{de-whiten}) > \text{MSE}(\text{no-whitening}).
        \]
        When the inequality of \eqref{newcond} is replaced with equality, then
        the two MSEs are equal.

        In the special case when $c < \lambda_{\min}$, i.e., $j^*=n$, 
        the condition in \eqref{newcond} reduces to $\bar{\lambda}:=\sum_{i=1}^n \lambda_i/n<1$.
        
        $(ii)$ Suppose $c > 1$. Then, 
        \[
        \text{MSE}(\text{de-whiten})\ge \text{MSE}(\text{no-whitening}),
        \]
        which is an equality if and only if $\lambda_i=c$, for all $i\le j^*$.
\end{theorem}

Theorem \ref{whitentrain2} offers valuable qualitative insights into the target whitening approach. In many practical situations, we have $\lambda_{\min} <1$ and $\bar{\lambda} <1$, as is the case with the Reacher, Swimmer, and Hopper training datasets (see Table \ref{tab:data}). Furthermore, test error is typically minimized with relatively small values of weight decay, corresponding to $c <1$. In such cases, Theorem \ref{whitentrain2} tells us that $\text{MSE}(\text{de-whiten})<\text{MSE}(\text{no-whitening})$, meaning that, under the UFM approximation, training with whitened targets  strictly reduces the training MSE. Conversely, if $\bar{\lambda} > 1$, as is the case for our CARLA 2D dataset (see Table \ref{tab:data}), Theorem \ref{whitentrain2} tells us that $\text{MSE}(\text{de-whiten}) > \text{MSE}(\text{no-whitening})$. In this scenario, training with whitened targets increases the training MSE. This duality highlights the critical role of $\bar{\lambda}$ in determining whether whitening improves or worsens training performance.

\textbf{Normalization}.
Another transformation closely related to whitening is normalization, which sets variances to 1 but leaves correlations intact:
\begin{equation}
    \bY^{nrm}=\bV^{-1/2} (\bY-\bYb),
\end{equation}
where $\bV$ is the diagonal matrix $\bV:=\text{diag}(\sigma_1^2,...,\sigma_n^2)$, and $\sigma_j^2$ denotes the variance of the $j$th target component. 
    
As we did previously for whitening, under the lens of the UFM theory, we can examine the effect of normalizing the targets with respect to the training MSE. To this end, we follow the same procedure as for whitening, but instead of using $\bY^{ZCA}$ we use $\bY^{nrm}$. Denote $\text{MSE}(\text{de-normalize})$ for the resulting de-normalized training MSE. Following this procedure, our next theorem provides a closed-form expression for $\text{MSE(de-normalization)}$.  

Let $\tilde{\lambda}_{\min}$, $\tilde{\lambda}_{\max}$ be the min and max eigenvalues of the sample correlation matrix of the target data given by $\bP=\bV^{-1/2} \bSigma \bV^{-1/2}$. Note that $0\le \tilde{\lambda}_{\min}\le 1$, and $\tilde{\lambda}_{\max}\ge 1$.

\begin{theorem} \label{stdtrain1}
\begin{equation}
\textnormal{MSE}(\textnormal{de-normalize}) = 
\begin{cases}
c \displaystyle \sum_{i=1}^n \lambda_i, &\text{if } 0<c < \tilde{\lambda}_{\min},
\\
\displaystyle \sum_{i=1}^n \lambda_i, &\text{if } c>\tilde{\lambda}_{\max}.
\end{cases}
\end{equation}
\end{theorem}

Let us now examine how $\text{MSE}(\text{de-normalize})$ compares to  $\text{MSE}(\text{de-whiten})$ and to the training MSE with the unprocessed targets. The next theorem follows directly from Theorems \ref{whitentrain1}-\ref{stdtrain1}.

\begin{theorem} \label{stdtrain2}
        $(i)$ Suppose $0<c<\min\{\lambda_{\min},\tilde{\lambda}_{\min}\}$. Then,
        \[
        \text{MSE}(\text{de-normalize})=\text{MSE}(\text{de-whiten}).
        \]
        Furthermore, if $\bar{\lambda}<1$, then $\text{MSE}(\text{de-normalize}) <  \text{MSE}$,
where $\text{MSE}$ is the training MSE using unprocessed targets.  If $\bar{\lambda}> 1$, then $\text{MSE}(\text{de-normalize}) >   \text{MSE}$. If  $\bar{\lambda}=1$, then $\text{MSE}(\text{de-normalize}) =   \text{MSE}$. 

(ii) Suppose $c>\tilde{\lambda}_{\max}$. Then,
\[
\text{MSE}(\text{de-normalize})=\text{MSE}(\text{de-whiten}).
\]
Furthermore, $\text{MSE}(\text{de-normalize}) \ge \text{MSE}.$
\end{theorem}

Theorem \ref{stdtrain2} provides important qualitative insights into the target normalization approach. In practical scenarios, such as with all of our 4 training datasets, we observe that $\lambda_{\min}<1$ ($\tilde{\lambda}_{\min}\le 1$ for any arbitrary correlation matrix), as indicated in Table \ref{tab:data}. Additionally, test error is typically minimized with small weight decay values, corresponding to $c<1$. Under these conditions, $\text{MSE}(\text{de-normalize})=\text{MSE}(\text{de-whiten})$, implying that under the UFM approximation, one can use either target whitening or normalization. As in Theorem \ref{whitentrain2}, Theorem \ref{stdtrain2} also highlights the crucial role of $\bar{\lambda}$ in determining whether normalization improves or worsens training performance. We also note that Theorem \ref{stdtrain2} does not provide a characterization for
the case $\min\{\lambda_{\min},\tilde{\lambda}_{\min}\} < c < \tilde{\lambda}_{\max}$.
We leave this as an open research problem. 

\begin{figure*}[htbp]
    \centering
    \begin{tabular}{c c c c c}
        & \textbf{Reacher} & \textbf{Swimmer} & \textbf{Hopper} & \textbf{CARLA 2D} \\
        \raisebox{1.5em}{\rotatebox{90}{\small \textbf{Training MSE}}} &
        \includegraphics[width=0.205\linewidth]{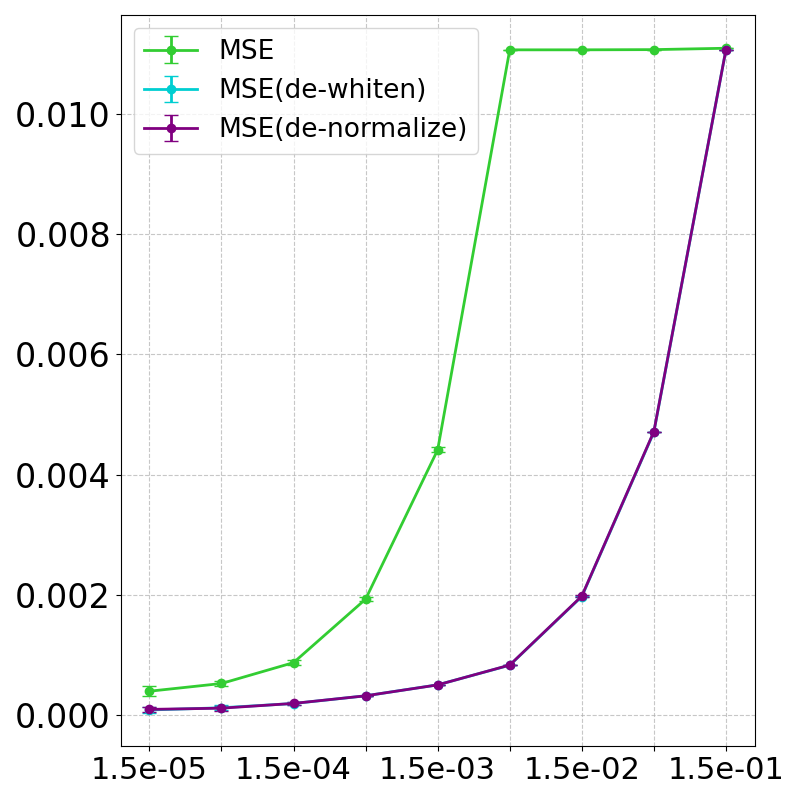} &
        \includegraphics[width=0.205\linewidth]{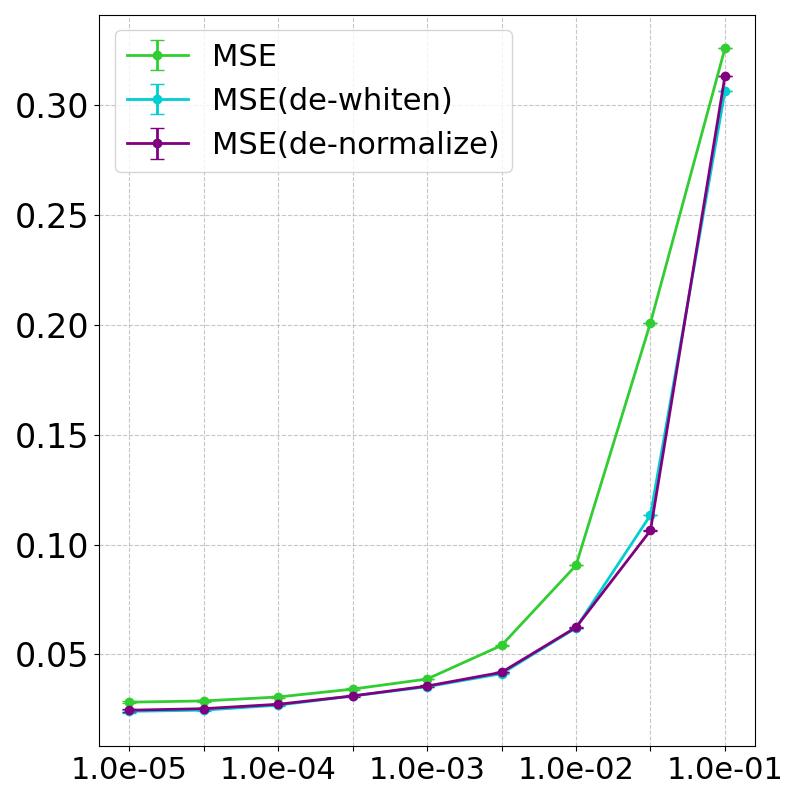} &
        \includegraphics[width=0.205\linewidth]{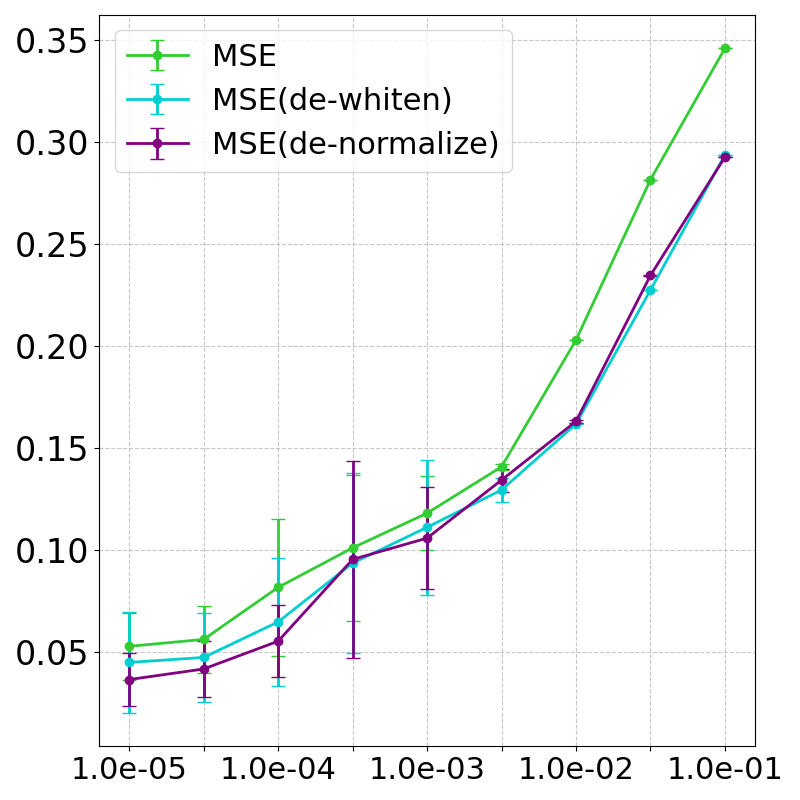} &
        \includegraphics[width=0.205\linewidth]{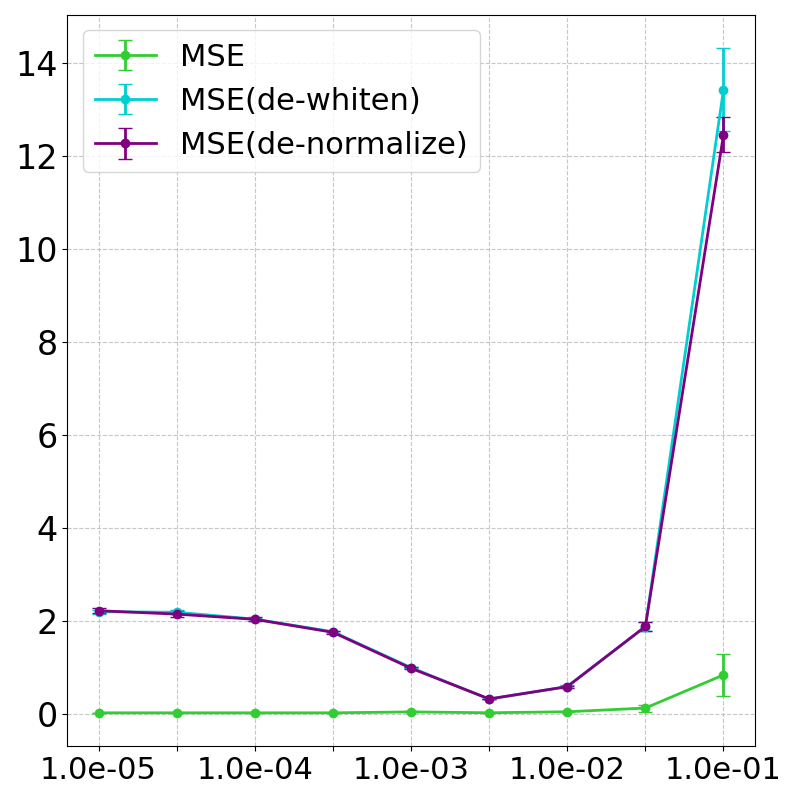} \\
        & $\lambda_{WD}$ & $\lambda_{WD}$ & $\lambda_{WD}$ & $\lambda_{WD}$
    \end{tabular}
    \caption{Comparison of the effect that target whitening and normalization have on training error for different weight decay values after training with the standard parameter-regularized loss function. The green curve (in short MSE) records the training error for different weight decay values after training with the original unprocessed targets.}
    \label{fig:case1}
\end{figure*}

\subsection{Experimental results: target whitening and normalization} \label{sec:5.1}

To validate the theoretical results, we conduct empirical experiments on three MuJoCo datasets and CARLA 2D. We adopt the same network structures as in Section \ref{sec:4.2}. For different weight decay values $\lambda_{WD}$, we evaluate three settings: (1) using the raw target $y$ as a baseline, (2) using the whitened targets, and (3) using the normalized targets.  When whitening or normalizing, we train the model with the transformed $y$ and then apply the inverse transformation (de-whitening or de-normalizing) to compute the training MSE.

Figure \ref{fig:case1} presents the experimental results, which turn out to closely align with the theory.  First, as the theory suggests, there is minimal difference in training MSE between whitening and normalizing. Second, the impact of whitening/normalizing on training MSE depends on the average variance over target dimensions, $\bar{\lambda}$. Whitening/normalizing reduces the training MSE when $\bar{\lambda} < 1$, as is the case in the three MuJoCo datasets. In contrast, when $\bar{\lambda} > 1$, as is the case in the CARLA 2D dataset, whitening/normalizing leads to an increase in training MSE. Once again, these empirical results are invariable to changes in the choice of architecture, cf. Figure \ref{fig:case1} with Figures \ref{fig:whiten_vs_norm_reacher}-\ref{fig:whiten_vs_norm_hopper} in Appendix \ref{additexpres}.  Moreover, the trends for test MSE align closely with those shown for training MSE in Figure \ref{fig:case1}, see Figure \ref{fig:test-mse-whitening} in Appendix \ref{additexpres}. These results not only validate the UFM theory but also provide practical insights for selecting appropriate data pre-processing strategies in machine learning pipelines.

\section{Conclusion} \label{conclsec}

This paper demonstrates the power of the Unconstrained Feature Model (UFM) as a theoretical tool for understanding neural multivariate regression tasks. By deriving and analyzing closed-form expressions for training mean-squared error (MSE), we explored two critical problems: multi-task versus single-task regression models and the benefits of target whitening and normalization. We found that for both of these problems, the UFM was applicable and provided novel qualitative insights. Our experiments then confirmed the correctness of the UFM predictions. Beyond these specific insights, the broader significance of this work lies in its demonstration of how the UFM can bridge theory and practice. 
The UFM’s ability to yield qualitative insights into key design choices, such as model architecture and data pre-processing strategies, establishes it as a valuable tool for advancing the design and optimization of DNNs. Additionally, developing tractable models within the UFM framework that extend to understanding and predicting generalization performance would be a valuable direction for further research.



\bibliography{example_paper}
\bibliographystyle{plain}
\clearpage

\appendix

\section{Experimental details}
\label{sec:a_exp}

\subsection{MuJoCo}
\label{appendix:mujoco}

The datasets Reacher and Swimmer are sourced from an open-source repository \citep{gallouedec2024jack} and consist of expert data generated by a policy trained using Proximal Policy Optimization (PPO)  \citep{schulman2017proximal}. For Hopper, the dataset is part of the D4RL benchmark \citep{fu2020d4rl}, which is widely recognized in offline reinforcement learning research. Table \ref{table:mujoco_mlp} provides a summary of all model hyper-parameters and experimental configurations used in Sections \ref{sec:4.2} and \ref{sec:5.1}. In all experiments, the models are trained until their weights converge. Additional details regarding the MuJoCo datasets and dataset-specific hyper-parameter settings are provided below.

\begin{table}[h!]
\caption{Hyper-parameter settings for experiments with weight decay on MuJoCo datasets.}
\label{table:mujoco_mlp}
\begin{center}
\begin{tabular}{cll}
\toprule
    & \textbf{Hyper-parameter} & \textbf{Value}  \\
\midrule
    & Number of hidden layers   & $3$ \\
  Model Architecture  & Hidden layer dimension & $256$  \\ 
    & Activation function & ReLU  \\ 
    & Number of linear projection layer ($\bW$) & 1 \\
\midrule
    & Epochs & 2e5, Reacher \\
    &   & 2e5, Swimmer \\
    &   & 4e4, Hopper\\
    & Batch size & 256 \\
    & Optimizer & SGD \\
Training    & Learning rate & 1e-2 \\
    & Seeds & 0, 1, 2 \\
    & Compute resources & NVIDIA A100 8358 80GB \\
    & Number of compute workers & 4 \\
    & Requested compute memory & 16 GB \\
    & Approximate average execution time & 5 hours\\    
\bottomrule
\end{tabular}
\end{center}
\end{table}

\textbf{MuJoCo environment descriptions}. We utilize expert data from previous work \citep{gallouedec2024jack, fu2020d4rl} for the Reacher, Swimmer, and Hopper environments. The Reacher environment features a two-jointed robotic arm, where the objective is to control the arm’s tip to reach a randomly placed target in a 2D plane. The Swimmer environment consists of a three-segment robot connected by two rotors, designed to propel itself forward as quickly as possible. Similarly, the Hopper environment is a single-legged robot with four connected body parts, aiming to hop forward efficiently. In all three cases, the robots are controlled by applying torques to their joints, which serve as the action inputs. To construct the datasets, an online reinforcement learning algorithm was used to train expert policies \citep{gallouedec2024jack,fu2020d4rl}. These expert policies were then deployed in the environments to generate offline datasets, consisting of state-action pairs where the states $\bx_i$ encompass the robot’s positions, angles, velocities, and angular velocities, while the targets $\byi$ correspond to the torques applied to the joints.

\textbf{Low data regime}. Training neural networks with expert state-action data using regularized regression is commonly known as \textit{imitation learning}. Following standard practices for MuJoCo environments \citep{tarasov2022corl}, we employ relatively small multi-layer perceptron (MLP) architectures. Since the goal in imitation learning is to achieve strong performance with minimal expert data, we train models using only a fraction of the available datasets. Specifically, we use 20 expert demonstrations (episodes) for Reacher, 1 for Swimmer, and 10 for Hopper, translating to datasets of 1,000, 1,000, and 10,000 samples respectively. Additionally, for each environment, we construct a validation set that contains $20\%$ of the size of the training data.
\looseness=-1

\subsection{CARLA}
The CARLA dataset is created by capturing the vehicle's surroundings through automotive cameras while a human driver controls the vehicle in a simulated urban environment \citep{Codevilla2018}. The recorded images represent the vehicle’s states $\bx_i$, and the expert driver’s control inputs, including speed and steering angles, are treated as the actions $\by_i \in[0, 85]\times  [-1, 1]$ in the dataset. A model trained on this data is expected to navigate the vehicle safely within the virtual environment.

For feature extraction from images, we use ResNet-18 \citep{he2016deep} as the backbone model. Since a large number of images is required to train a robust feature extractor from visual inputs \citep{he2016deep, sun2017revisiting}, the entire dataset is utilized for training. To adapt the ResNet architecture, which was initially designed for classification tasks, to a regression setting, we replace the final classification layer with a fully connected layer that outputs continuous values corresponding to the targets. The experimental setup for CARLA is detailed in Table \ref{table:carla_resnet18}.

\begin{table}[h!]
\caption{Hyper-parameters of ResNet for CARLA dataset.}
\label{table:carla_resnet18}
\begin{center}
\begin{tabular}{cll}
\toprule
    & \textbf{Hyper-parameter} & \textbf{Value}  \\
\midrule
    & Backbone of hidden layers & ResNet18 \\
  
 Architecture    & Last layer hidden dim & $512$  \\ 
\midrule
    & Epochs & 100 \\
    & Batch size & 512 \\
    & Optimizer & SGD \\
    & Momentum & 0.9 \\
Training    & Learning rate & 0.001 \\
     
    & Seeds & 0, 1\\
    & Compute resources & NVIDIA A100 8358 80GB \\
    & Number of compute workers & 8 \\
    & Requested compute memory & 200 GB \\
    & Approximate average execution time & 42 hours\\
    
\bottomrule
\end{tabular}
\end{center}
\end{table} 

\subsection{Inter-task correlations}

Our chosen datasets exhibit a range of task correlations. Table \ref{table:datasets} below includes the Pearson correlation coefficient between the $i$-th and $j$-th target components for $i\neq j$. When the target dimension is 2, there is one correlation value between the two target components; when the target dimension is 3, there are three correlation values between the three target components. From Table \ref{table:datasets}, we observe that the target components in CARLA 2D and Reacher are nearly uncorrelated, whereas those in Hopper and Swimmer exhibit stronger correlations. This demonstrates that multi-task learning's advantages in our experiments are not solely attributable to high-correlation scenarios. While strongly correlated tasks do provide a ``best-case" benchmark, our results also highlight settings where task relationships are weak - underscoring multi-task learning's ability to leverage even limited shared structure.  

\begin{table}[h!]
\caption{Overview of datasets employed in our analysis.}
\label{table:datasets}
\begin{center}
\begin{tabular}{cccccc}
\toprule
\textbf{Dataset} & \textbf{Data Size} & \textbf{Input Type} & \textbf{Target Dimension} $n$ & \textbf{Target Correlation} \\
\midrule
Swimmer & 1K & raw state & 2 & -0.244 \\

Reacher & 10K & raw state & 2 & -0.00933 \\

Hopper & 10K & raw state & 3  & [-0.215, -0.090, 0.059]\\
\midrule
CARLA 2D & 600K & RGB image & 2 & -0.0055\\

\bottomrule
\end{tabular}
\end{center}
\end{table} 

\section{Additional experimental results} \label{additexpres}

\subsection{Summary of train MSE empirical results across a broad range of architectures}

The UFM formulation depends on the idea that the nonlinear feature extractor is flexible enough to be capable of approximating any function. To explore what happens when this part of the network has dramatically reduced or increased capacity, we cover in depth how the training of the networks impacts the experimental results of Section \ref{sec:4.2} and Section \ref{sec:5.1}. 

To provide further analysis of how the size of neural networks influences the empirical results of Figure \ref{fig:multi-task-vs-n-single-task-case1} and Figure \ref{fig:case1}, we have repeated our experiments across multiple architectures, covering a broad range of widths and depths. For this ablation, we selected the MuJoCo environments of Reacher, Swimmer, and Hopper. The results summarized in Figures \ref{fig:arch_vs_tasks_reacher}-\ref{fig:arch_vs_tasks_hopper} are organized from left to right and top to bottom based on the number of parameters of the neural networks. We find that, regardless of the architecture, our experimental results align with the predictions of Theorem \ref{singlevsmulti} regarding multi-task and single-task models. Similarly, in Figures \ref{fig:whiten_vs_norm_reacher}-\ref{fig:whiten_vs_norm_hopper}, we summarize experiments across a broad range of architectures to test the claims of Theorem \ref{stdtrain2} regarding training with whitened and normalized targets as opposed to training with raw targets. The results remain consistent.

Thus, our choice of datasets with input data ranging from raw robotic states (vectors) to images and training networks ranging from shallow MLPs to the deep and wide ResNet architecture, strikes a balance, and demonstrates the soundness of our empirical findings for low and high capacity networks.

\newpage

\begin{figure}[htbp]
  \centering
  \includegraphics[width=1\linewidth]{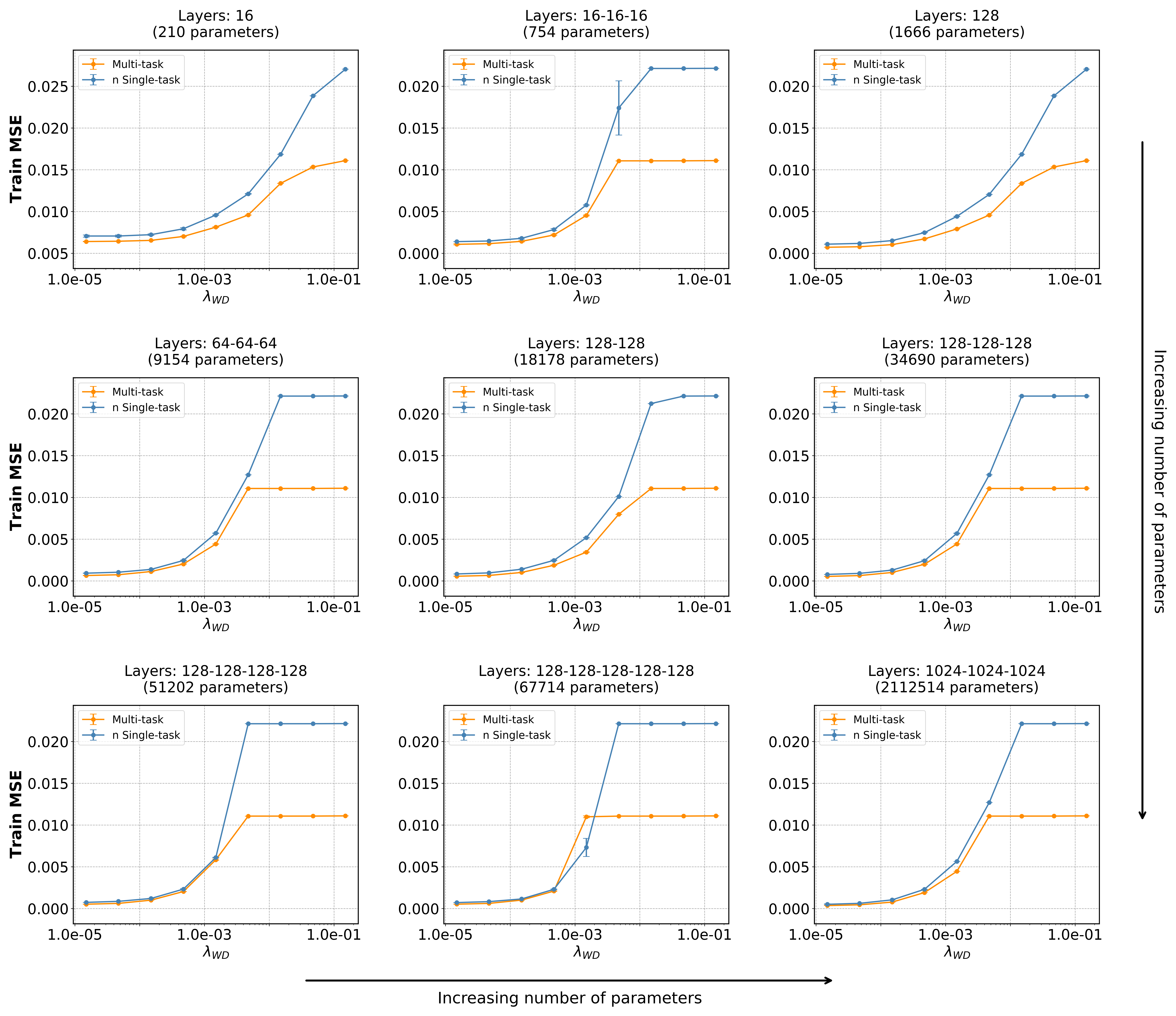}
  \caption{\textbf{Effect of network architecture on multi‑ vs.\ single‑task training: Reacher.} Comparison of the training error of a single multi-task model with that of multiple single task models for different weight decay values after training with the standard parameter-regularized loss function across different architectures. The architectures are denoted by by their layer sizes (input and output layers omitted for simplicity). The number of parameters increases from left to right and from top to bottom.}
  \label{fig:arch_vs_tasks_reacher}
\end{figure}

\begin{figure}[htbp]
  \centering
  \includegraphics[width=1\linewidth]{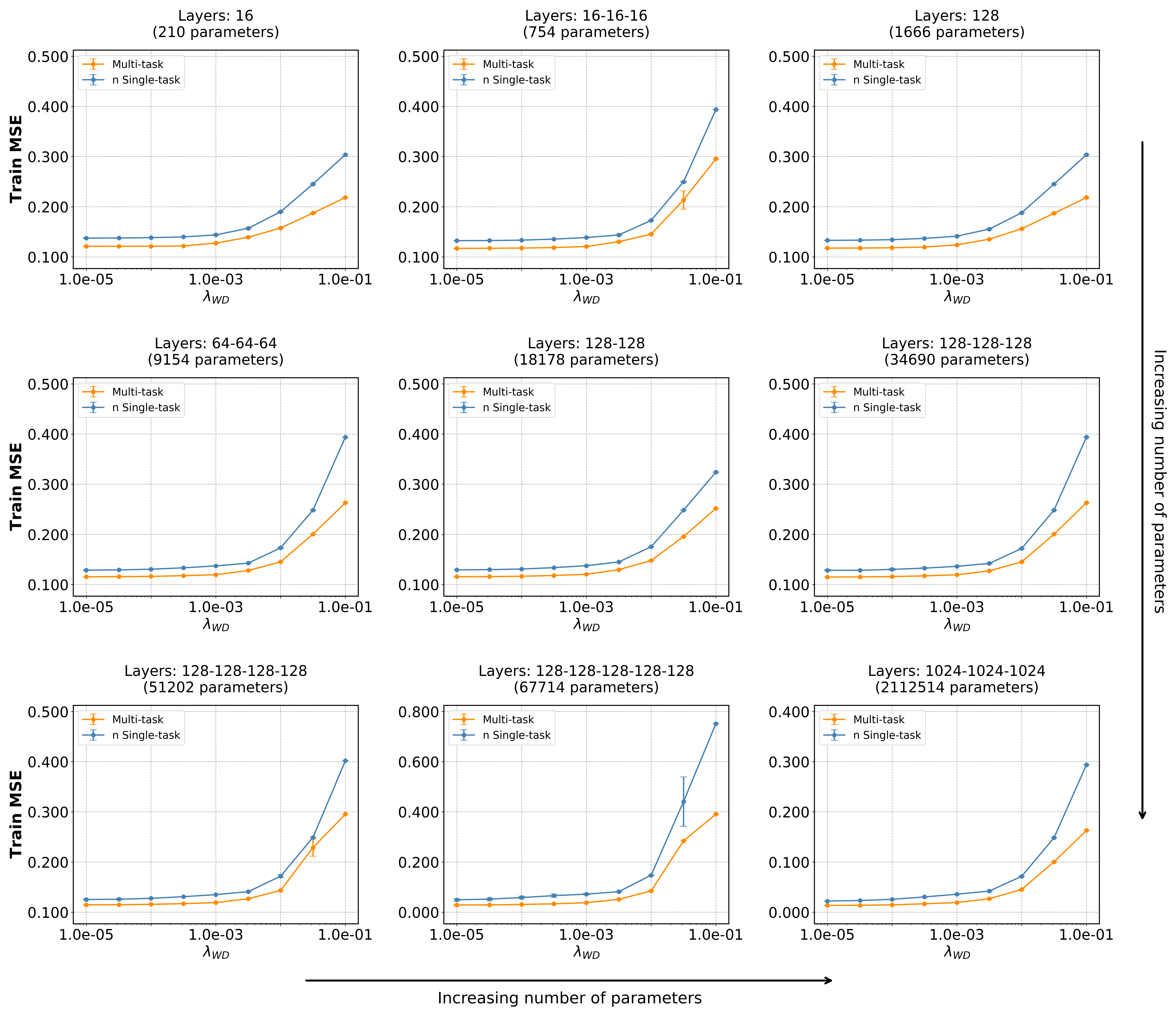}
  \caption{\textbf{Effect of network architecture on multi‑ vs.\ single‑task training: Swimmer.} Comparison of the training error of a single multi-task model with that of multiple single task models for different weight decay values after training with the standard parameter-regularized loss function across different architectures. The architectures are denoted by by their layer sizes (input and output layers omitted for simplicity). The number of parameters increases from left to right and from top to bottom.}
  \label{fig:arch_vs_tasks_swimmer}
\end{figure}

\begin{figure}[htbp]
  \centering
  \includegraphics[width=1\linewidth]{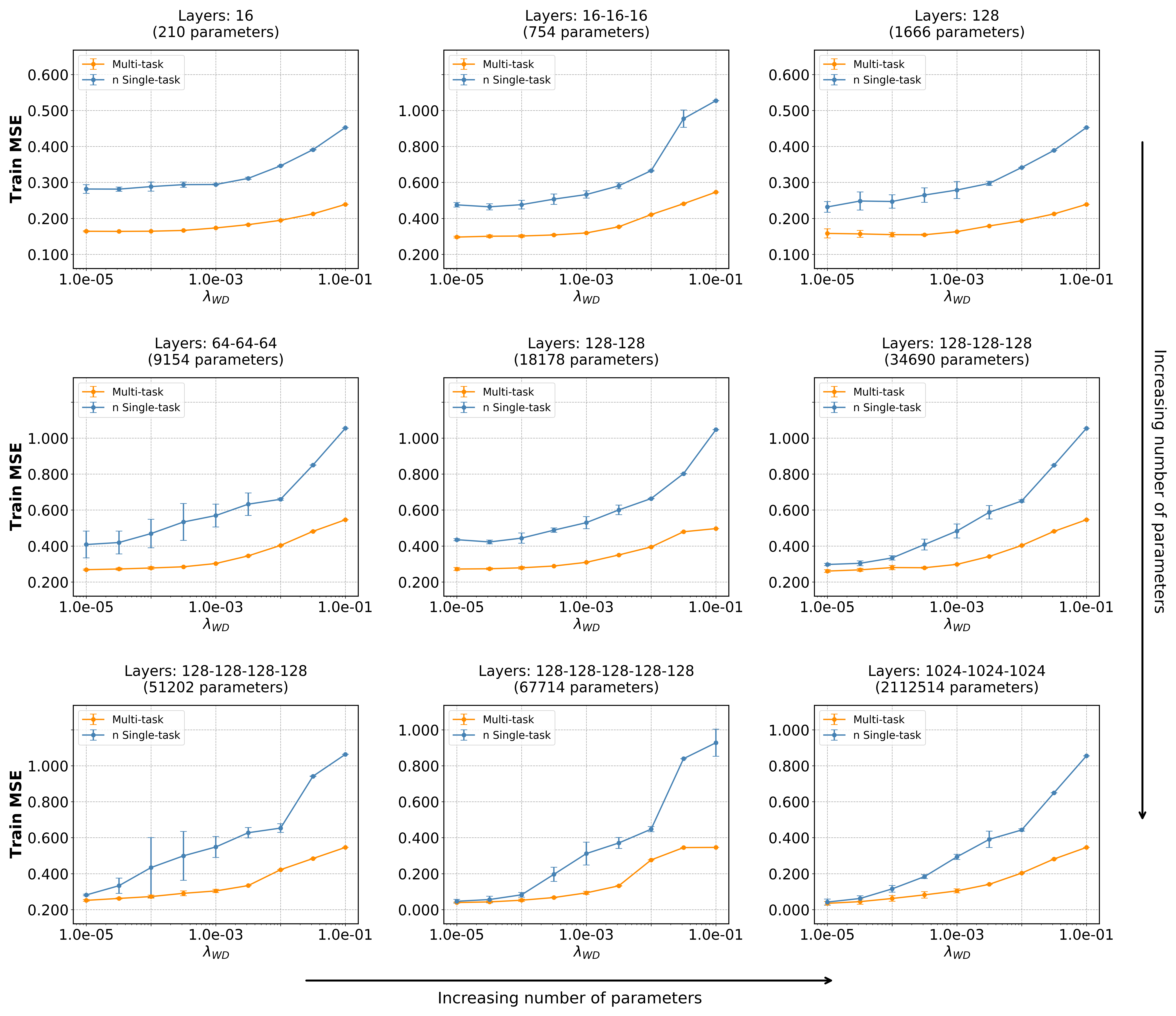}
  \caption{\textbf{Effect of network architecture on multi‑ vs.\ single‑task training: Hopper.} Comparison of the training error of a single multi-task model with that of multiple single task models for different weight decay values after training with the standard parameter-regularized loss function across different architectures. The architectures are denoted by by their layer sizes (input and output layers omitted for simplicity). The number of parameters increases from left to right and from top to bottom.}
  \label{fig:arch_vs_tasks_hopper}
\end{figure}

\newpage  

\begin{figure}[htbp]
  \centering
  \includegraphics[width=1\linewidth]{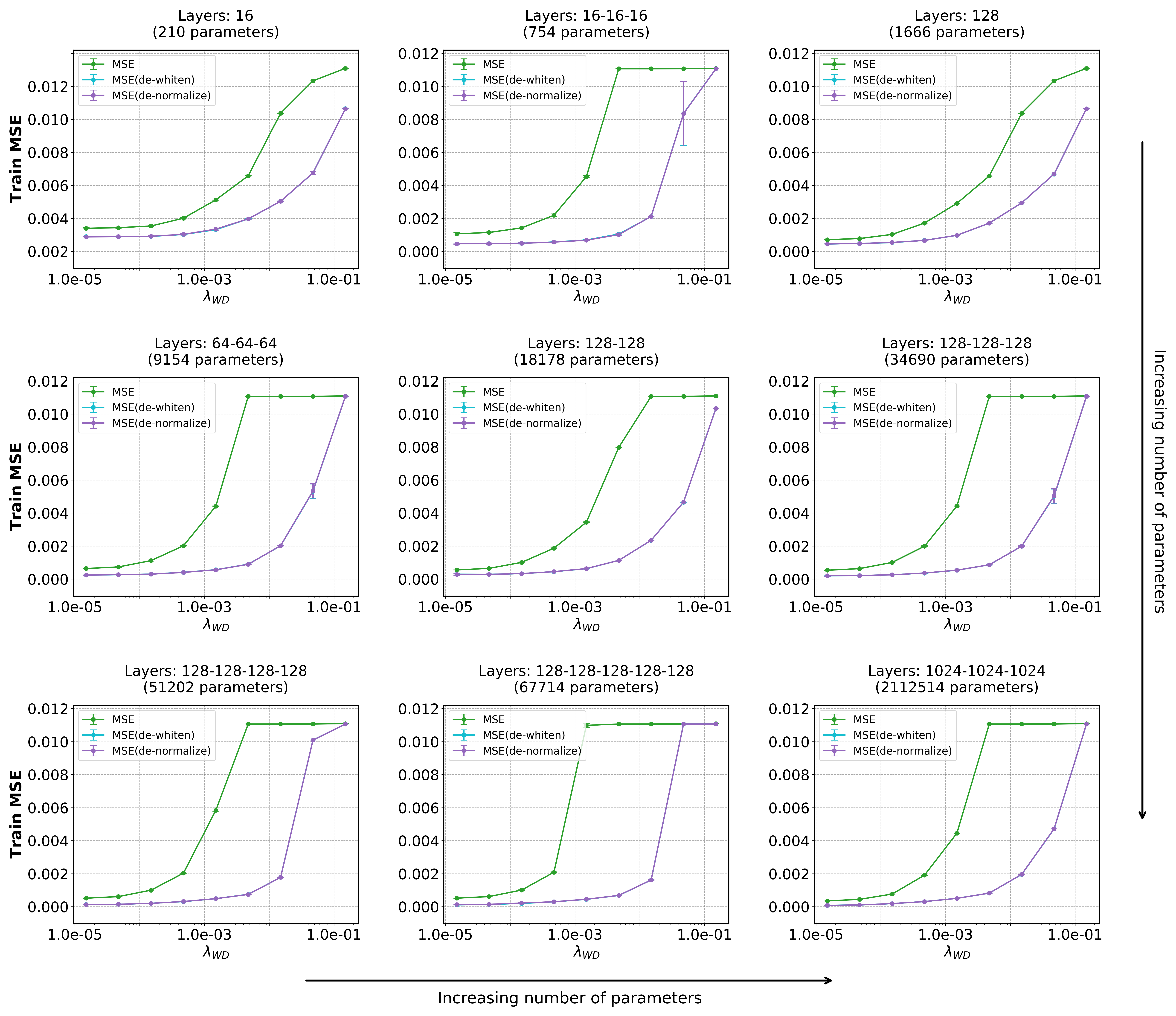}
  \caption{\textbf{Whitening vs.\ normalization vs.\ raw targets (Swimmer).}  
           Training‑MSE comparisons for the Reacher environment. Comparison of the effect that target whitening and normalization have on training error for different weight decay values after training with the standard parameter-regularized loss function across different architectures. The green curve (in short MSE) records the training error for different weight decay values after training with the original unprocessed targets. The architectures are denoted by their layer sizes (input and output layers omitted for simplicity). The number of parameters increases from left to right and from top to bottom.}
  \label{fig:whiten_vs_norm_reacher}
\end{figure}

\begin{figure}[htbp]
  \centering
  \includegraphics[width=1\linewidth]{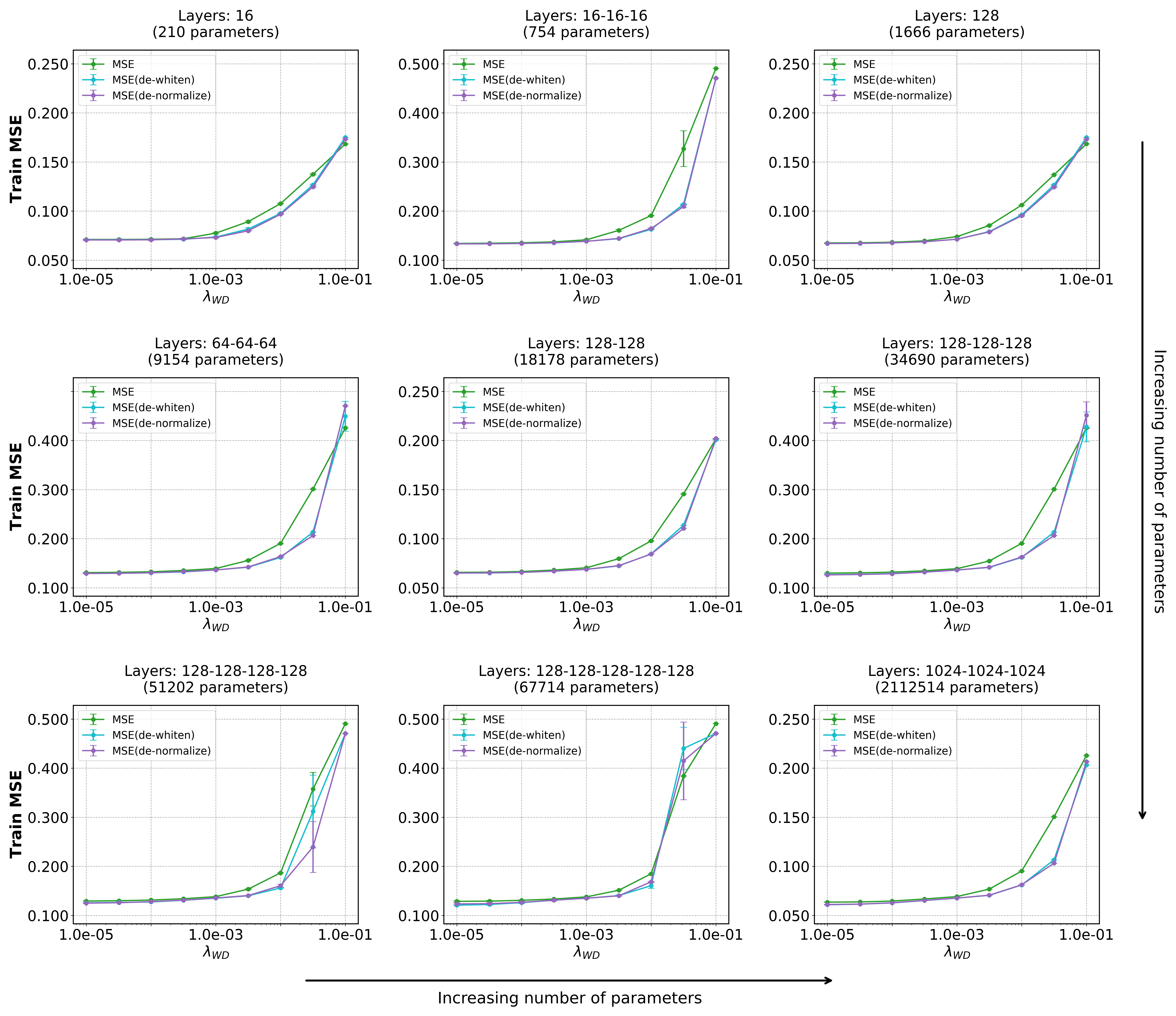}
  \caption{\textbf{Whitening vs.\ normalization vs.\ raw targets (Reacher).}  
           Training‑MSE comparisons for the Reacher environment. Comparison of the effect that target whitening and normalization have on training error for different weight decay values after training with the standard parameter-regularized loss function across different architectures. The green curve (in short MSE) records the training error for different weight decay values after training with the original unprocessed targets. The architectures are denoted by their layer sizes (input and output layers omitted for simplicity). The number of parameters increases from left to right and from top to bottom.}
  \label{fig:whiten_vs_norm_swimmer}
\end{figure}

\begin{figure}[htbp]
  \centering
  \includegraphics[width=1\linewidth]{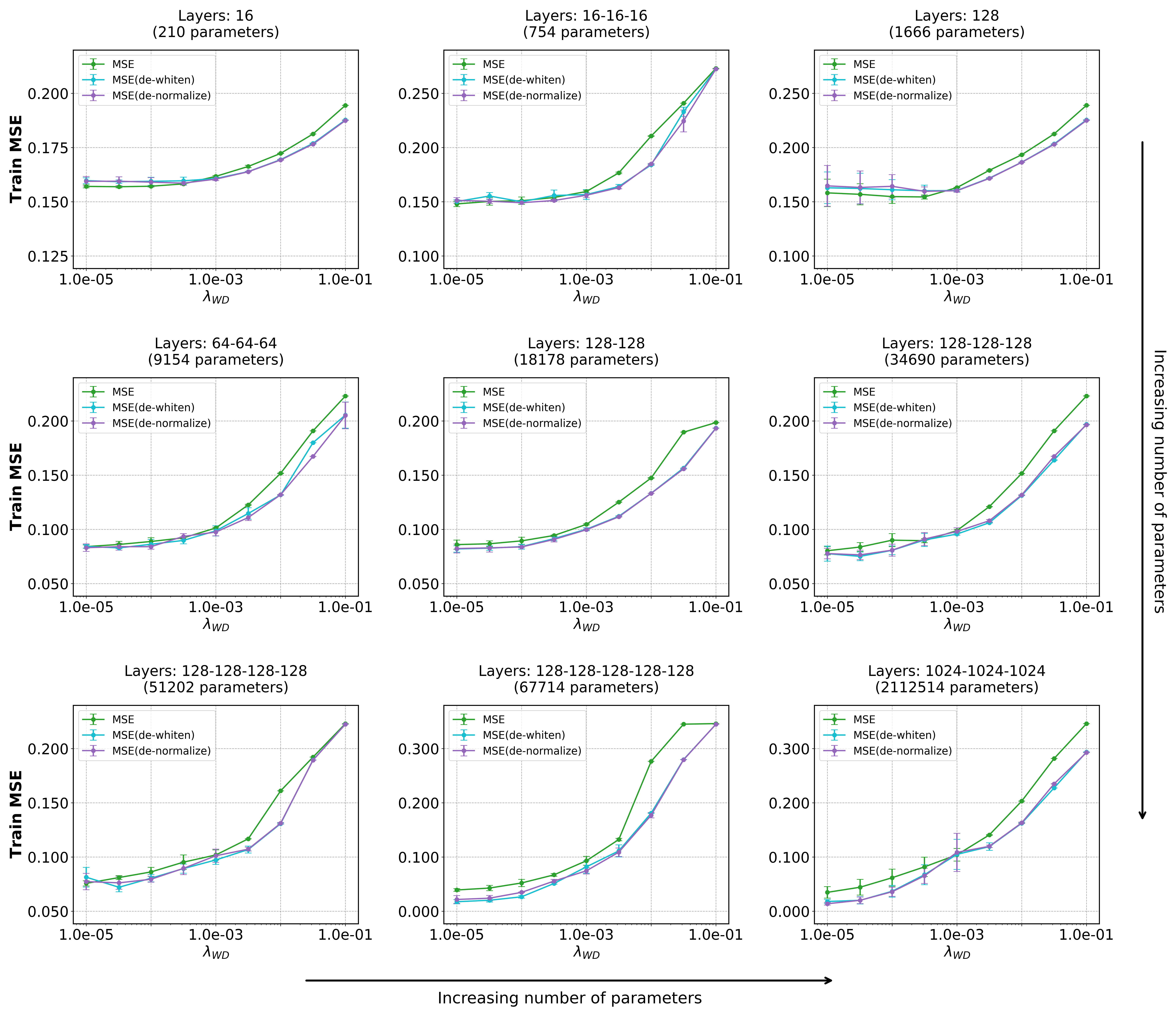}
  \caption{\textbf{Whitening vs.\ normalization vs.\ raw targets (Hopper).}  
           Training‑MSE comparisons for the Reacher environment. Comparison of the effect that target whitening and normalization have on training error for different weight decay values after training with the standard parameter-regularized loss function across different architectures. The green curve (in short MSE) records the training error for different weight decay values after training with the original unprocessed targets. The architectures are denoted by their layer sizes (input and output layers omitted for simplicity). The number of parameters increases from left to right and from top to bottom.}
  \label{fig:whiten_vs_norm_hopper}
\end{figure}
\clearpage

\subsection{Summary of test MSE empirical results}

We have conducted extensive experiments to evaluate test MSE, complementing our theoretical and empirical analysis on training MSE. Regarding the inclusion of test results, these serve two purposes. One, to illustrate that even with zero training MSE, e.g., in the unregularized cases, modest weight decay can still improve generalization, and another, to provide preliminary evidence that our theoretical findings on training MSE do not lead to impractical model behavior. The results are summarized in the figures below. 

For the vast majority of cases, the trends for test MSE align closely with those shown for training MSE in Figures \ref{fig:multi-task-vs-n-single-task-case1} and \ref{fig:case1}:

\begin{itemize}
    \item Figure \ref{fig:test-mse-multi}: Multi-task learning consistently achieves lower test MSE than single-task learning across all datasets and weight decay values, supporting the intuition that, by taking on account task dependencies and learning shared patterns through weights and representation sharing in a single model, multi-task regression improves generalization. For an attempt to explain this theoretically using generalization bounds, we refer the reader to the discussion that follows the derivation of \eqref{2ndgenbound} and \eqref{3rdgenbound} in Appendix \ref{genboundappendix}.
    \item Figure \ref{fig:test-mse-whitening}: The difference between whitening and normalization is minor, in line with Figure \ref{fig:case1} of the main body. The effect depends on the average eigenvalue of the covariance matrix. When the average eigenvalue is $<1$ (e.g., the three MuJoCo datasets), whitening/normalization generally reduces test MSE, except for very small weight decay values. When the average eigenvalue is $>1$ (e.g., CARLA 2D), whitening/normalization increases test MSE, reflecting the same trend observed for training MSE.
\end{itemize}

In summary, the test MSE results reinforce the robustness of our findings: modest regularization improves generalization, multi-task learning offers consistent benefits, and whitening/normalization effects align with the spectral properties of the target data covariance.

\begin{figure*}[htbp]
    \centering
    \begin{tabular}{c c c c c}
        & \textbf{Reacher} & \textbf{Swimmer} & \textbf{Hopper} & \textbf{CARLA 2D} \\
        \raisebox{1.5em}{\rotatebox{90}{\small \textbf{test MSE}}} &
        \includegraphics[width=0.205\linewidth]{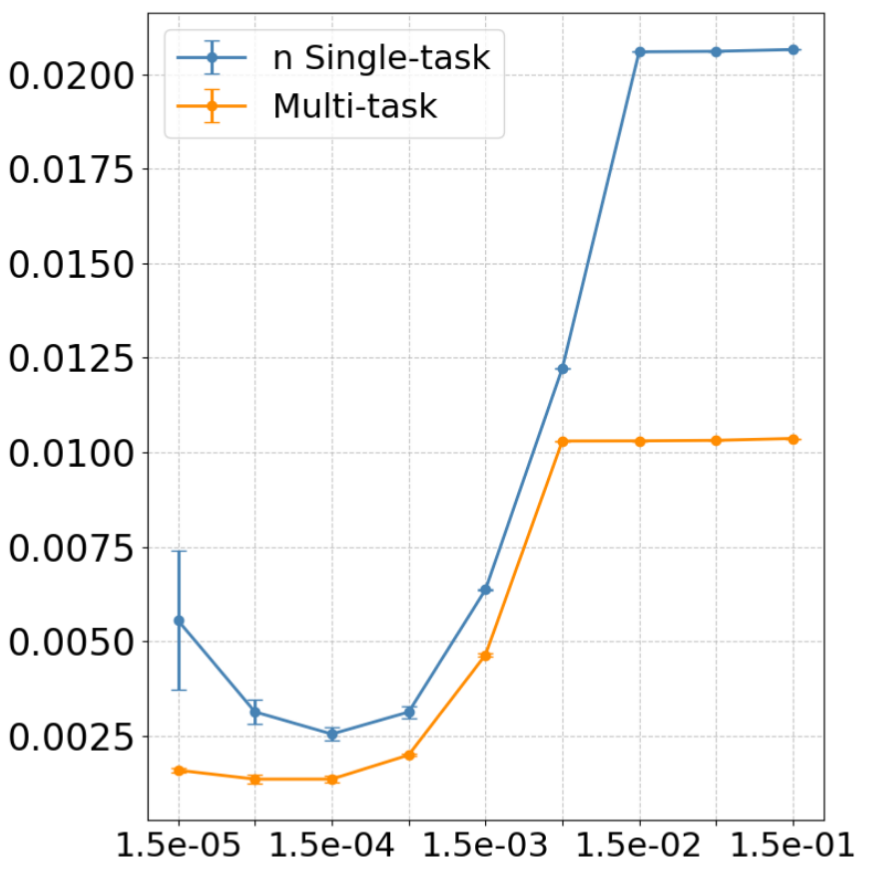} &
        \includegraphics[width=0.205\linewidth]{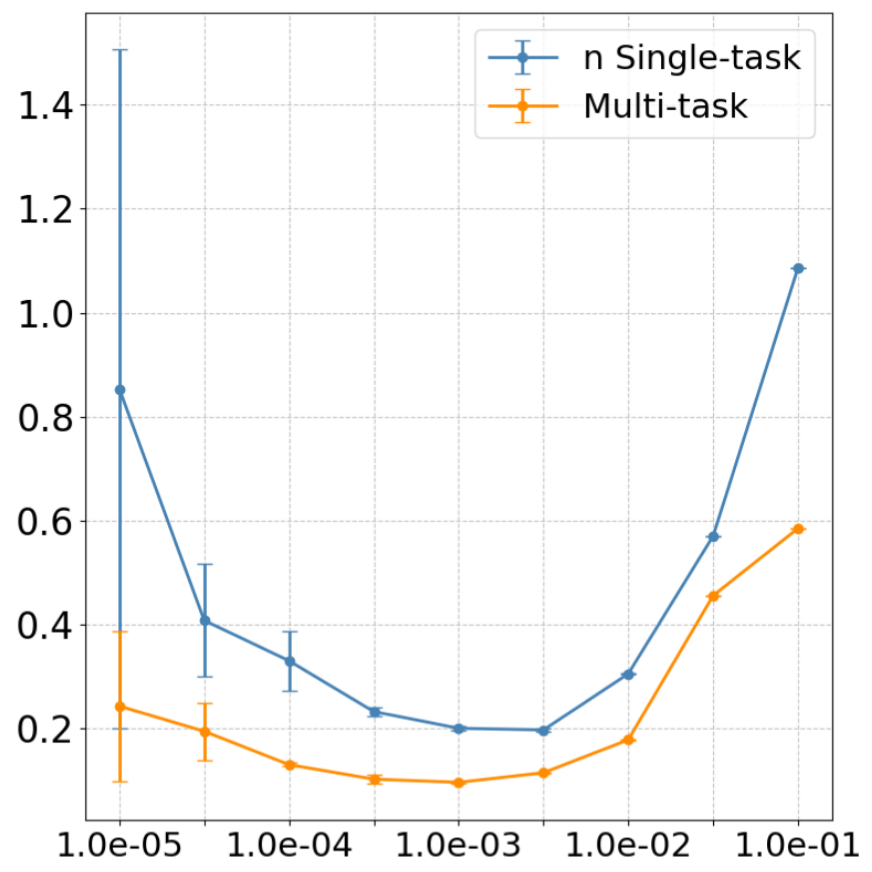} &
        \includegraphics[width=0.205\linewidth]{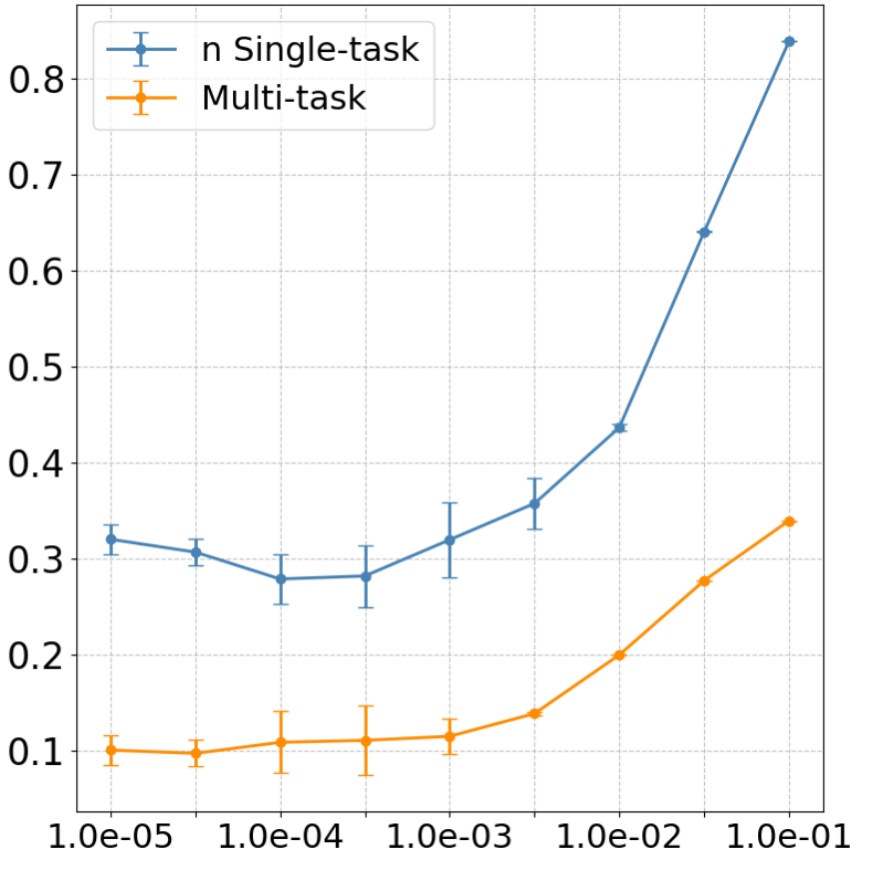} &
        \includegraphics[width=0.205\linewidth]{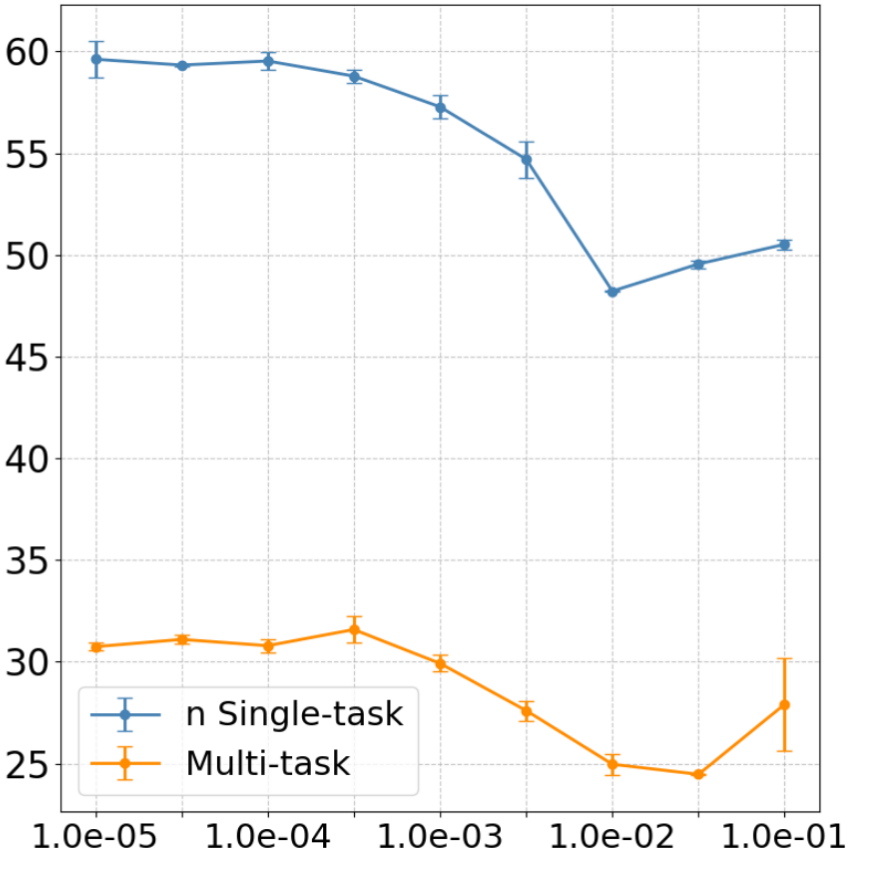}
        \\
        & $\lambda_{WD}$ & $\lambda_{WD}$ & $\lambda_{WD}$ & $\lambda_{WD}$
    \end{tabular}
    \caption{Comparison of the test error of a single multi-task model with that of multiple single task models for different weight decay values after training with the standard parameter-regularized loss function. Values shown as mean±std across random seeds.}
    \label{fig:test-mse-multi}
\end{figure*}

\begin{figure*}[htbp]
    \centering
    \begin{tabular}{c c c c c}
        & \textbf{Reacher} & \textbf{Swimmer} & \textbf{Hopper} & \textbf{CARLA 2D} \\
        \raisebox{1.5em}{\rotatebox{90}{\small \textbf{test MSE}}} &
        \includegraphics[width=0.205\linewidth]{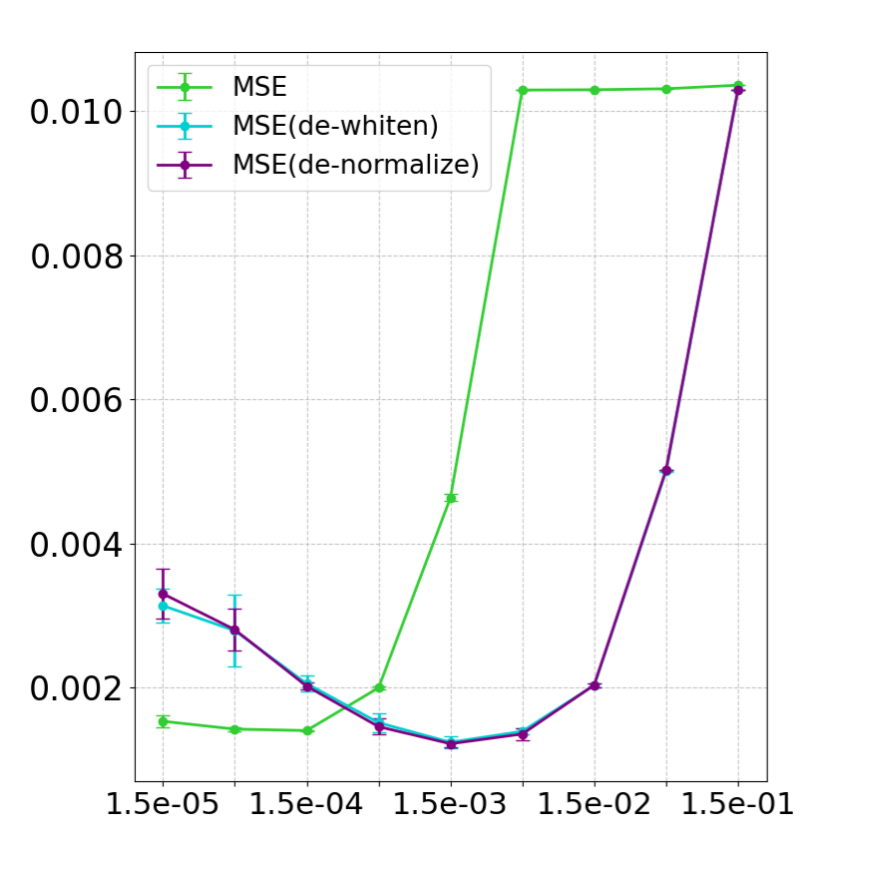} &
        \includegraphics[width=0.205\linewidth]{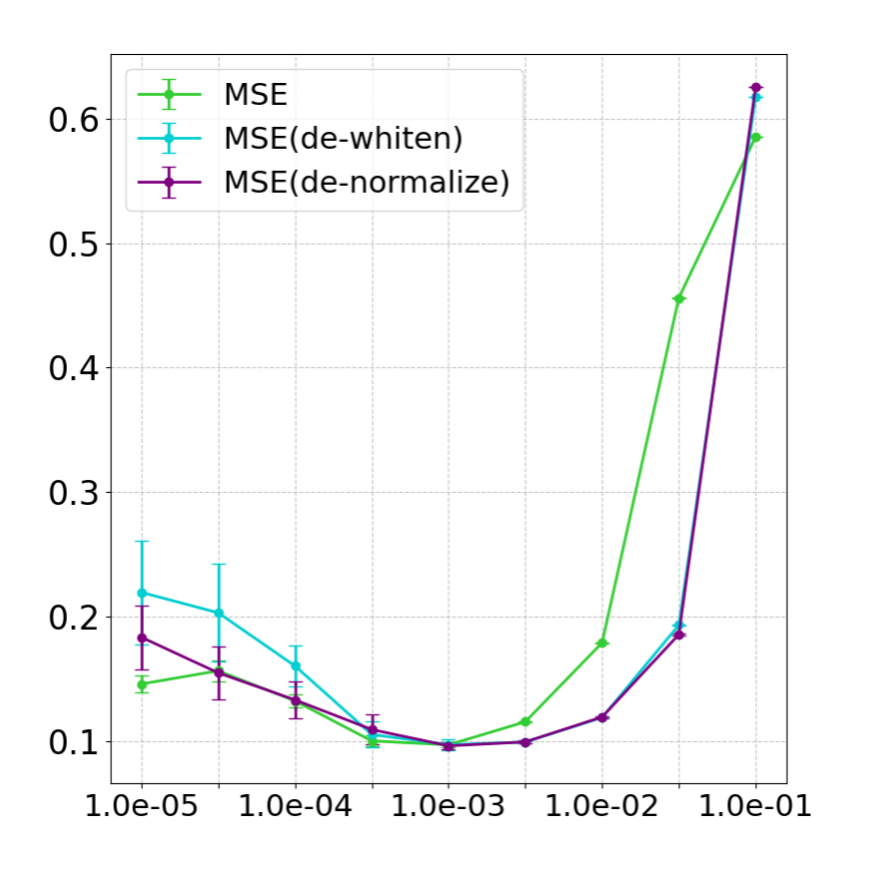} &
        \includegraphics[width=0.205\linewidth]{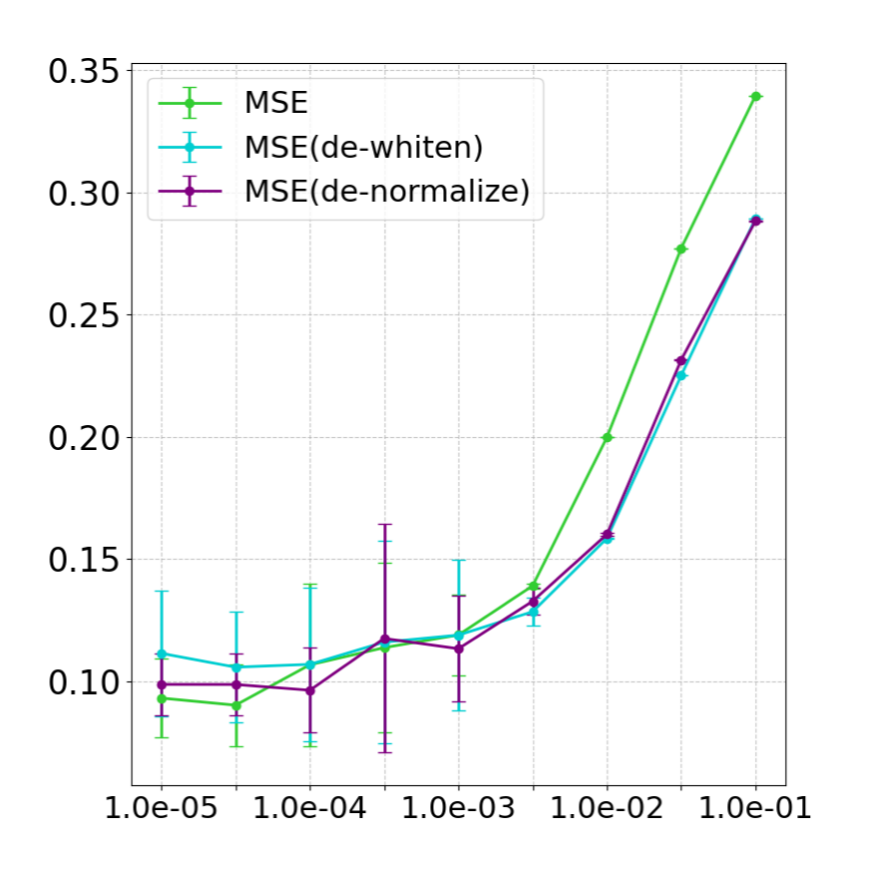} &
        \includegraphics[width=0.205\linewidth]{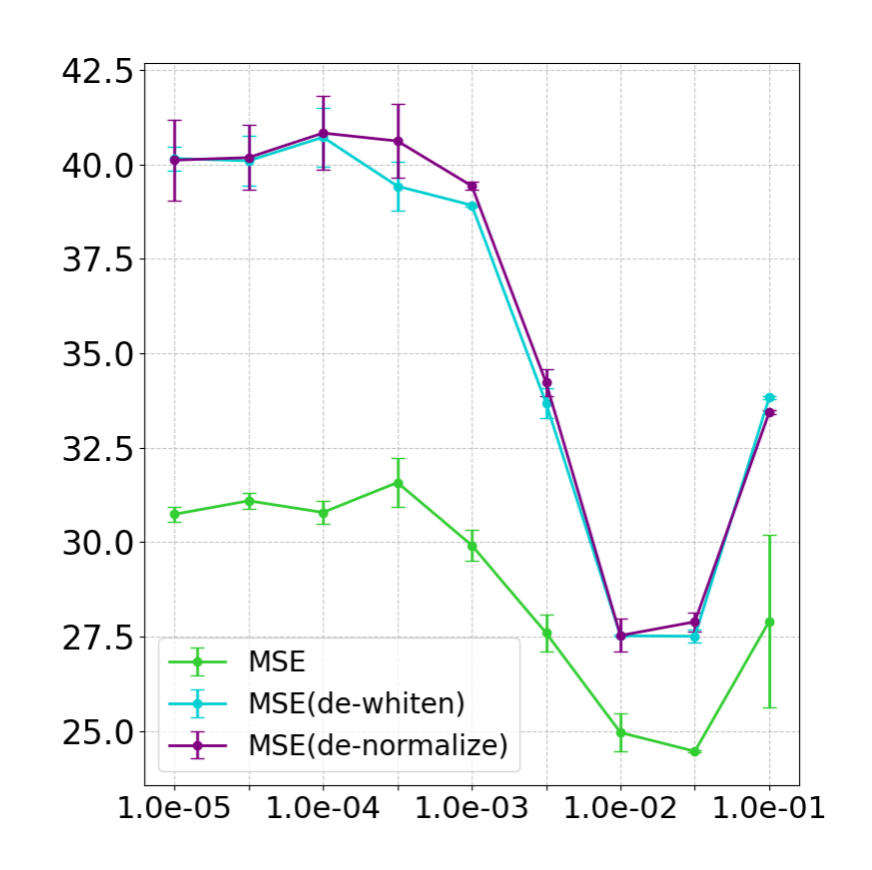}
        \\
        & $\lambda_{WD}$ & $\lambda_{WD}$ & $\lambda_{WD}$ & $\lambda_{WD}$
    \end{tabular}
    \caption{Comparison of the test error of a single multi-task model with that of multiple single task models for different weight decay values after training with the standard parameter-regularized loss function. Values shown as mean±std across random seeds.}
    \label{fig:test-mse-whitening}
\end{figure*}

\section{Proof of Theorem \ref{gendim}} \label{proofthmsecd}

\begin{proof}[Proof of Theorem \ref{gendim}]
    Let $\tilde{\bY}=\bY-\bYb=U \tilde{\bSigma} V^T$ denote the compact SVD of $\tilde{\bY}$, where $U\in \mathbb{R}^{n\times n}$ is orthogonal, $V\in \mathbb{R}^{M\times n}$ is semi-orthogonal, and $\tilde{\bSigma}\in \mathbb{R}^{n\times n}$ is diagonal containing the singular values $\eta_1\ge \eta_2\ge \cdots \ge \eta_n > 0$. 
    Let $(\bH^*, \bW^*, \bb^*)$ be a global minimum of \eqref{formofloss}. By Lemma B.1\citep{zhou2022optimization}, the associated MSE  is 
\begin{equation} \label{minMSE}
\text{MSE}(\bH^*, \bW^*, \bb^*)=\frac{1}{M} \sum_{i=1}^n ([\eta_i-\sqrt{Mc}]_{+}-\eta_i)^2.
\end{equation}
Furthermore, using the SVD of $\tilde{\bY}=U \tilde{\bSigma} V^T$,
\[
\bSigma = \frac{\tilde{\bY}\tilde{\bY}^T}{M} =
U \frac{\tilde{\bSigma}}{\sqrt{M}} V^T V \frac{\tilde{\bSigma}}{\sqrt{M}} U^T= U \left[\frac{\tilde{\bSigma}}{\sqrt{M}}\right]^2 U^T,
\]
from which we have $\bSigma^{1/2} = U \frac{\tilde{\bSigma}}{\sqrt{M}} U^T$. This further yields
\[
\sqrt{M} [ \bSigma^{1/2} - \sqrt{c} \bI_n ] = U \tilde{\bSigma} U^{T}-U \sqrt{ M c} \bI_n U^T
\]
Since $U^T=U^{-1}$,
\begin{equation} \label{eigendec}
\sqrt{M} [\bSigma^{1/2}-\sqrt{c} \bI_n]=U \left[\tilde{\bSigma} -\sqrt{M c}\bI_n\right] U^{-1},
\end{equation}
which implies that the matrices $\sqrt{M} [\bSigma^{1/2}-\sqrt{c} \bI_n]$ and $\tilde{\bSigma} -\sqrt{M c}\bI_n$ are similar. As a result, they have the same eigenvalues.
The $n\times n$ matrix on the left-hand side of \eqref{eigendec} has eigenvalues given by $\sqrt{M \lambda_i}-\sqrt{M c}$, $i=1,..., n$, where $\lambda_i$ is the $i$-th eigenvalue of $\bSigma$, 
whereas the $n\times n$ matrix on the right-hand side of \eqref{eigendec} has eigenvalues
$\eta_i-\sqrt{M c}$, $i=1,..., n$. Since the eigenvalues in these two sets are both arranged in descending order, we have
\begin{equation} \label{eigeneq}
\sqrt{\lambda_i} = \frac{\eta_i}{\sqrt{M}}, \qquad \text{for all } i=1,..., n.
\end{equation}
Correspondingly, by \eqref{minMSE} and \eqref{eigeneq} we obtain
\begin{align*}
\text{MSE}(\bH^*, \bW^*, \bb^*)
&= \frac{1}{M} \sum_{i=1}^n ([\eta_i-\sqrt{Mc}]_{+}-\eta_i)^2
\\
&= \frac{1}{M} \sum_{i=1}^n ([\sqrt{M\lambda_i}-\sqrt{Mc}]_{+}-\sqrt{M\lambda_i})^2
\\
&= j^{*} c + \sum_{i=j^{*}+1}^n \lambda_i
\end{align*}
as desired.
When $n=1$, $\bSigma$ is simply the scalar $\sigma^2$, which together with the equation above completes the proof.
\end{proof}

Next, we give a simple upper bound of the form \eqref{actbound} as a corollary. The bound is explicitly given as the minimum of the  MSEs when $c<\lambda_{\min}$ ($j^{*}=n$) and $c>\lambda_{\max}$ ($j^{*}=0$) respectively. 

\begin{corollary} \label{simplebound}
    \begin{equation} \label{actbound}
    \textnormal{MSE}(\bH, \bW, \bb)\le \min\left\{nc, \sum_{i=1}^n \lambda_i\right\}.
    \end{equation}
\end{corollary}

\begin{proof}
    Note that since $c\le \lambda_i$, for all $i\le j^{*}$, and $c>\lambda_i$, for all $i>j^{*}$, we have that 
    \[
    \text{MSE}(\bH, \bW, \bb)=j^{*} c +\sum_{i=j^{*}+1}^n \lambda_i < j^{*} c + \sum_{i=j^{*}+1}^n c = n c,
    \]
    \[
     \text{MSE}(\bH^*, \bW^*, \bb^*)=j^{*} c +\sum_{i=j^{*}+1}^n \lambda_i = \sum_{i=1}^{j^{*}} c +  \sum_{i=j^{*}+1}^n \lambda_i \le \sum_{i=1}^{j^{*}} \lambda_i +  \sum_{i=j^{*}+1}^n \lambda_i = \sum_{i=1}^n \lambda_i.
    \]
    The desired result readily follows.
\end{proof}

\section{Proof of Theorem \ref{singlevsmulti}} \label{th4.2}

\begin{proof}[Proof of Theorem \ref{singlevsmulti}]
We begin by using the Schur-Horn theorem to establish some relations between the eigenvalues $\lambda_1,\ldots,\lambda_n$ of $\bSigma$ and the diagonal elements (variances) $\sigma_1^2,\ldots,\sigma_n^2$ of $\bSigma$. Recall that both $\lambda_1,\ldots,\lambda_n$ and $\sigma_1^2,\ldots,\sigma_n^2$ are arranged in descending order. 
By the Schur-Horn theorem, the vector containing the diagonal elements of $\bSigma$ is majorized by the vector that contains the ordered eigenvalues of $\bSigma$, i.e.,
\begin{align}
\sum_{i=1}^k \lambda_i &\ge \sum_{i=1}^k \sigma_i^2,\qquad \mbox{ for all } k=1,\ldots,n-1, \label{schur1}
\\
\sum_{i=1}^n \lambda_i&=\sum_{i=1}^n \sigma_i^2. \label{schur2}
\end{align}
From \eqref{schur1} and \eqref{schur2} we have 
\begin{equation} \label{schur3}
    \lambda_1 \geq \sigma_1^2 \geq \sigma_2^2 \geq \cdots \geq \sigma_n^2 \geq \lambda_n.
\end{equation}
and also the following inequalities for the tail partial sums:
\[
    \sum_{i=k}^n \lambda_i \leq \sum_{i=k}^n \sigma_i^2,\qquad \mbox{ for all } k=1,\ldots,n.
    \label{schur4}
\]
Fix $0 < c < \tilde{c}$. We will consider 6 cases:

\textbf{Case I:} Suppose $c < \tilde{c} < \lambda_n$. By \eqref{schur3}, we have that $\tilde{c}<\sigma_n^2$. Thus, by Theorem \ref{gendim} and Corollary \ref{msensingle}, the two MSEs are given by 
\[
nc = \text{MSE}(\text{multi}, c)<\text{MSE}(\text{n-single}, \tilde{c})=n \tilde{c}.
\]
Observe that when $c=\tilde{c}$, the two MSEs are equal to $nc$.

\textbf{Case II:} Suppose $\tilde{c} \ge c > \lambda_1$. By \eqref{schur3}, we have that $\tilde{c}>\sigma_1^2$. Thus, by Theorem \ref{gendim} and Corollary \ref{msensingle}, the two MSEs are given by  
\[
\text{MSE}(\text{multi},c)=\sum_{i=1}^n \lambda_i = \sum_{i=1}^n \sigma_i^2 = \text{MSE}(\text{n-single},\tilde{c}).
\]

\textbf{Case III:} 
Suppose $c<\lambda_n$, $\tilde{c}>\lambda_1$. From Theorem \ref{gendim} and Corollary \ref{msensingle}, the difference of the two MSEs is given by
\[
\textnormal{MSE}(\text{multi},c) - \textnormal{MSE}(\text{n-single},\tilde{c}) = nc-\sum_{i=1}^n \lambda_i =\sum_{i=1}^n (c-\lambda_i)< 0,
\]
since $c<\lambda_n\le \lambda_i$, for all $i=1,...,n$.

\textbf{Case IV:} 
Suppose $\lambda_n<c<\lambda_1<\tilde{c}$. From Theorem \ref{gendim} and Corollary \ref{msensingle}, the difference of the two MSEs is given by
\begin{align*}
\textnormal{MSE}(\text{multi},c) - \textnormal{MSE}(\text{n-single},\tilde{c}) 
=j^{*}c-\sum_{i=1}^{j^{*}}\lambda_i
=
\sum_{i=1}^{j^{*}}(c-\lambda_i)<0,
\end{align*}
since $c<\lambda_1$ and $c\le \lambda_i$, for all $i\le j^{*}$.

\textbf{Case V:}  
Suppose $c<\lambda_n<\tilde{c}<\lambda_1$. From Theorem \ref{gendim} and Corollary \ref{msensingle}, the difference of the two MSEs is given by
\begin{align*}
\textnormal{MSE}(\text{multi},c) - \textnormal{MSE}(\text{n-single},\tilde{c}) 
&=n c -k^{*}\tilde{c}-\sum_{i=k^{*}+1}^{n}\sigma_i^2
\\
&= \sum_{i=k^{*}+1}^n (c-\sigma_i^2)+k^{*}(c-\tilde{c})
\\
&< \sum_{i=k^{*}+1}^n (\lambda_i-\sigma_i^2)+k^{*}(c-\tilde{c})
\\
&=
\sum_{i=1}^{k^{*}} (\sigma_i^2-\lambda_i) + k^{*}(c-\tilde{c}) < 0,
\end{align*}
where the first inequality holds since $c < \lambda_n\le \lambda_i$, for all $i=1,...,n$, and the last inequality is due to \eqref{schur1} and since $c<\tilde{c}$.

\textbf{Case VI:} 
Suppose $\lambda_n < c \le \tilde{c} < \lambda_1$. Recall that $j^{*} := \max \{j: \lambda_j\ge c\}$ and $k^{*} := \max \{j: \sigma_j^2\ge \tilde{c}\}$. We will consider three subcases.

\textbf{VIa):} Suppose $j^{*} < k^{*}$. From Theorem \ref{gendim} and Corollary \ref{msensingle}, the difference of the two MSEs is given by
\begin{align}
&\textnormal{MSE}(\text{multi},c) - \textnormal{MSE}(\text{n-single},\tilde{c}) \nonumber
\\
&=  \sum_{i=j^{*}+1}^n \lambda_i -\sum_{i=k^{*}+1}^n \sigma_i^2 - (k^{*} - j^{*})c + k^{*}(c-\tilde{c}) \nonumber
\\
&=\sum_{i=j^{*} +1}^{k^{*}} \lambda_i + \sum_{i=k^{*}+1}^n \lambda_i - \sum_{i=k^{*} +1}^n \sigma_i^2 - 
(k^{*}-j^{*}) c + k^{*} (c-\tilde{c}) \nonumber
\\
&= \sum_{i=k^{*} +1}^n (\lambda_i-\sigma_i^2) + \sum_{i=j^{*} +1}^{k^{*}} (\lambda_i-c) + k^{*}(c-\tilde{c})<0.  \label{MSEdifference1}
\end{align}
The first sum in \eqref{MSEdifference1} is non-positive by \eqref{schur4}. 
The second sum is strictly negative since 
$\lambda_i<c$, for all $i>j^{*}$. Therefore, $\text{MSE}(\text{multi}, c) < \text{MSE}(\text{n-single}, \tilde{c})$. 

\textbf{VIb):} Suppose $k^{*} \le j^{*}$. From Theorem \ref{gendim} and Corollary \ref{msensingle}, the difference of the two MSEs is given by
\begin{align}
&\text{MSE}(\text{multi},c)-\text{MSE}(\text{n-single},\tilde{c}) \nonumber 
\\
&=  \sum_{i=k^{*}+1}^{j^{*}} (c-\sigma_i^2) + \sum_{i=j^{*} +1}^n (\lambda_i-\sigma_i^2) + k^{*}(c-\tilde{c})
\nonumber \\
&=  \sum_{i=k^{*}+1}^{j^{*}} (c-\sigma_i^2) + \sum_{i=1}^{j^{*}} (\sigma_i^2 -\lambda_i) + k^{*}(c-\tilde{c})
\nonumber \\
&=
\sum_{i=k^{*} +1}^{j^{*}} (c-\sigma_i^2)+\sum_{i=1}^{k^{*}} (\sigma_i^2-\lambda_i)
+\sum_{i=k^{*} +1}^{j^{*}} (\sigma_i^2-\lambda_i) + k^{*}(c-\tilde{c})
\nonumber \\
&\leq
\sum_{i=k^{*} +1}^{j^{*}} (\lambda_i -\sigma_i^2)+\sum_{i=1}^{k^{*}} (\sigma_i^2-\lambda_i) \label{inequality1}
+\sum_{i=k^{*} +1}^{j^{*}} (\sigma_i^2-\lambda_i) + k^{*}(c-\tilde{c})
\\
&=
\sum_{i=1}^{k^{*}} (\sigma_i^2-\lambda_i) + k^{*}(c-\tilde{c}) \le 0,
\label{inequality2}
\end{align}
where the inequality in \eqref{inequality1} holds due to the fact that $c\le \lambda_i$, for all $i\le j^{*}$, and the inequality \eqref{inequality2} is due to \eqref{schur1}.
Therefore, $\text{MSE}(\text{multi}, c) \le \text{MSE}(\text{n-single}, \tilde{c})$. 
\end{proof}

\begin{remark}
    \textit{Applicability without L2-regularization}: Our analysis relies on L2-regularization to yield non-trivial closed-form results, that is the UFM-approximation of training MSE in Theorem \ref{gendim} holds when the UFM-regularization constant $c>0$. If $c=0$, it is easy to see that, for any $n\times d$ matrix $\bW$ with full-rank $n$, considering the set of $\bH$ that satisfy $||\bW \bH-\bY||_F^2=0$, gives:
    \[
    \bH = \bW^{+} \bY + (\bI_d-\bW^{+} \bW) \bZ, 
    \]
    where $\bW^{+}$ is the pseudoinverse of $\bW$ and $\bZ$ is any $d\times M$ matrix, and this is the well-known solution of the standard least squares problem. Thus, comparing multi-task ($c=0$) with single-tasks ($\tilde{c}=0)$ is trivial as both training MSEs are identical and equal to zero. Such a UFM-inspired approximation agrees with classical notions of overfitting.
\end{remark}

\section{Proofs of Theorems \ref{whitentrain1} and \ref{whitentrain2}} \label{disc}

Before delving into the proof of Theorems \ref{whitentrain1} and \ref{whitentrain2}, we provide a brief introduction to the motivation and key concepts relevant to whitening.  

In statistical analysis, whitening (or sphering) refers to a common pre-processing step to transform random variables to orthogonality. A whitening transformation (or sphering transformation) is a linear transformation that converts a random vector with a known covariance matrix into a new random vector of the same dimension and with covariance matrix given by the identity matrix. 
Orthogonality among random vectors greatly simplifies multivariate data analysis both from a computational as well as from a statistical standpoint. Whitening is employed mostly in pre-processing but is also part of modeling, see for instance \citep{hao2015sparsifying, zuber2009gene}.

Due to rotational freedom there are infinitely many whitening transformations. All produce orthogonal but different sphered random variables. To understand differences between whitening transformations, and to select an optimal whitening procedure for a particular situation, the work of \citep{kessy2018optimal}, provided an overview of the underlying theory and discussed several natural whitening procedures. For example, they identified PCA whitening as the unique procedure that maximizes the compression of all components of the unprocessed vector in each component of the sphered vector. 
Of a particular interest for our work is \textit{ZCA whitening}, ZCA standing for zero-phase component analysis. Rather than dimensionality reduction and data compression, ZCA whitening is useful for retaining maximal similarity between the unprocessed and the transformed variables.

\begin{definition} \label{zcadef}
    ZCA whitening employs the sphering matrix
    $\bW^{ZCA}=\bSigma^{-1/2}$, i.e.,
    \[
    \bY^{ZCA}:=\bSigma^{-1/2} (\bY-\bYb).
    \]
\end{definition}

We make the following observations:

\begin{enumerate}
    \item Clearly, writing $\bW^{ZCA}=\mathbf{Q}_1 \bSigma^{-1/2}$, where $\mathbf{Q}_1$ is an orthogonal matrix, $\bW^{ZCA}$ satisfies 
        \[
        \bW^{ZCA} \bSigma (\bW^{ZCA})^T=\bI_n,
        \]
    thus leading to new (of the infinitely many) whitening transformations. In fact, with $\mathbf{Q}_1=\bI_n$, ZCA whitening is the unique sphering method with a symmetric whitening matrix. 
    \item 
    Breaking the rotational invariance by investigating the cross-covariance between unprocessed (centered) and sphered targets is key to identifying the optimal whitening transformations.
    The sample \textit{cross-covariance} between $\bY^{ZCA}$ and $\bY$ is given by
    \[
    \mathbf{\Phi}:=\frac{\bY^{ZCA} (\bY-\bYb)^T}{M}=\bW^{ZCA} \frac{(\bY-\bYb)(\bY-\bYb)^T}{M}=\bW^{ZCA} \bSigma=\mathbf{Q}_1 \bSigma^{1/2}.
    \]
    Note that $\mathbf{\Phi}$ is in general not symmetric, unless $\mathbf{Q}_1=\bI_n$. Using targets and their (ZCA)-whitened counterparts, the least squares objective is minimized when the trace of the cross-covariance is maximized, e.g., \citep{kessy2018optimal}[eq. (13) and (14)],
    \[
    \frac{1}{M}||\bY^{ZCA}-(\bY-\bYb)||_F^2= n-2\text{tr}(\mathbf{\Phi}) + \sum_{i=1}^n \sigma_i^2= n - 2 \text{tr}(\mathbf{Q}_1 \bSigma^{1/2}) + \sum_{i=1}^n \sigma_i^2. 
    \]
    It can be shown that the minimum of the latter is attained at $\mathbf{Q}_1=\bI_n$. Therefore, not only ZCA whitening is the unique sphering method with a symmetric whitening matrix. It is also the optimal whitening apporach identified by evaluating the objective function of the total squared distance between the unprocessed and the whitened targets, computed from the cross-covariance $\mathbf{\Phi}$.
\end{enumerate}

To summarize, ZCA whitening is the unique procedure used with the aim of making the transformed targets as similar as possible to the unprocessed targets, which is appealing since in many applications it is desirable to remove dependencies with minimal additional adjustments. 

We reformulate Theorem \ref{whitentrain1} to also include information about the structure of the global minima, and consequently the target predictions regarding those.

\begin{theorem} \label{whitentrain}
 Let $c:=\lH \lW$. Any global minimum $(\tilde{\bH}, \tilde{\bW})$ of the regularized UFM-loss
\begin{equation} \label{whiteobj}
\frac{1}{2M} ||\bW \bH -\bY^{ZCA}||_F^2 + \frac{\lH}{2M} ||\bH||_F^2+\frac{\lW}{2}||\bW||_F^2,
\end{equation}
where $\lH$ and $\lW$ are non-negative regularization parameters, takes the following form.

If $0<c<1$, then for any semi-orthogonal matrix $\bR$,
\begin{equation} \label{otherprops}
\tilde{\bW}  =  \left( {\frac{\lH}{\lW}} \right)^{1/4} \tilde{\bA}^{1/2}\bR, \hspace{.2in}
\tilde{\bH}  =   \sqrt{\frac{\lW}{\lH}} \tilde{\bW}^{T}  \bY^{ZCA},
\end{equation}
\[
\tilde{\bW} \tilde{\bH}=(1-\sqrt{c}) \bY^{ZCA},
\]
where $\tilde{\bA}=(1-\sqrt{c}) \bI_n$.

If $c> 1$, then $(\tilde{\bH}, \tilde{\bW})=(\mathbf{0}, \mathbf{0})$. 

Furthermore,
\[
\textnormal{MSE}(\textnormal{de-whiten}) = 
\begin{cases}
c \displaystyle \sum_{i=1}^n \lambda_i, &\text{ if } c < 1,
\\
 \displaystyle \sum_{i=1}^n  \lambda_i, &\text{ if } c \geq 1,
\end{cases}
\]
where $\lambda_i$ is the $i$-th eigenvalue (in descending order) of the original sample covariance matrix $\bSigma$.
\end{theorem}

Before giving the proof, let us first discuss the nature of $\tilde{\bW}$ and $\tilde{\bH}$ when whitening is applied. 
In light of \eqref{otherprops}, properties that hold are:
\begin{enumerate}
    \item The rows of $\tilde{\bW}$ are orthogonal (due to $\tilde{\bA}=(1-\sqrt{c}) \bI_n$ being diagonal).
    \item The rows of $\tilde{\bW}$ are equinorm. More specifically, 
    \[
    ||\tilde{\bw}_j||_2^2=\lH \left(\frac{1}{\sqrt{c}}-1\right), \qquad j=1,...,n.
    \]
    \item The angles between the columns of $\tilde{\bH}$ are equal to angles between the whitened $\tilde{\by}_i$'s, i.e.,
    \[
    \tilde{\bH}^T \tilde{\bH} = \lW \left(\frac{1}{\sqrt{c}}-1\right) (\bY^{ZCA})^T \bY^{ZCA}.
    \]
\end{enumerate}

\begin{proof}[Proof of Theorem \ref{whitentrain}]
It is easily seen that 
\[
M^{-1} \bY^{ZCA} (\bY^{ZCA})^T=\bI_n.
\]
Thus, the covarance matrix for the whitened targets is the $n \times n$ identity matrix. 
\\
\textbf{Case $c<1$:}
\\
By \citep{andriopoulos2024prevalence}[Theorem 4.1], we have that any global minimum $(\tilde{\bH}, \tilde{\bW)}$ for \eqref{whiteobj}
takes the form:
\[
\tilde{\bW}  =  \left( {\frac{\lH}{\lW}} \right)^{1/4} \tilde{\bA}^{1/2}\bR, \hspace{.2in}
\tilde{\bH}  =   \sqrt{\frac{\lW}{\lH}} \tilde{\bW}^{T}  \bY^{ZCA},
\]
where $\tilde{\bA}=(1-\sqrt{c}) \bI_n$. Note that, under whitening $j^{*}=\max\{j: c\le 1\}=n$. The optimal predictions after whitening (but before de-whitening) are
\[
\tilde{\bW} \tilde{\bH}=\tilde{\bA}^{1/2} \bR \bR^T \tilde{\bA}^{1/2} \bY^{ZCA}
 =(1-\sqrt{c}) \bY^{ZCA}.
\]
Our final de-whitened predictions satisfy
\[
\hat{\bY}= [\bSigma^{1/2}] \tilde{\bW} \tilde{\bH} + \bYb=  (1-\sqrt{c}) (\bY-\bYb)+\bYb.
\]
Therefore,
\[
\hat{\bY}-\bY = -\sqrt{c} (\bY-\bYb),
\]
and
\[
\frac{(\hat{\bY}-\bY) (\hat{\bY}-\bY)^T}{M}=  c \frac{(\bY-\bYb) (\bY-\bYb)^T}{M}=c \bSigma.
\]
The result for $\textnormal{MSE}(\text{de-whiten})$ readily follows by taking traces in both sides.
\\
\textbf{Case $c \geq 1$:}
\\
By \citep{andriopoulos2024prevalence}[Theorem 4.1], we have that the only global minimum is $(\tilde{\bH}, \tilde{\bW})=(\mathbf{0}, \mathbf{0})$. 
Therefore,  
\[
\hat{\bY}=\bYb,
\]
and 
\[
\frac{(\hat{\bY}-\bY) (\hat{\bY}-\bY)^T}{M}=\frac{(\bY-\bYb) (\bY-\bYb)^T}{M}=\bSigma.
\]
The result for $\textnormal{MSE}(\text{de-whiten})$ readily follows by taking traces in both sides.
\end{proof}

\begin{proof}[Proof of Theorem \ref{whitentrain2}]
    
Suppose $0<c\le 1$. By Theorem \ref{gendim} and Theorem \ref{whitentrain1},
    \begin{align*}
    \text{MSE}(\text{de-whiten})-\text{MSE}(\text{multi})
    &=
    c \sum_{i=1}^n \lambda_i - j^{*} c - \sum_{i=j^{*}+1}^n \lambda_i
    \nonumber 
    \\
    &=c \sum_{i=1}^{j^{*}} \lambda_i + c \sum_{i=j^{*}+1}^{n} \lambda_i - j^{*} c - \sum_{i=j^{*}+1}^n \lambda_i
    \nonumber \\
    &=c \left[\sum_{i=1}^{j^{*}} \lambda_i -j^{*}\right] + (c-1) \sum_{i=j^{*}+1}^n \lambda_i<0
    \end{align*}
   if and only if
   \[
   \sum_{i=1}^{j^*} \lambda_i-j^* < c^{-1} (1-c) \sum_{i=j^*+1}^{n} \lambda_i,
   \]
which is condition \eqref{newcond} as postulated in the assumptions of Theorem \ref{whitentrain2}(i).

Suppose $c>1$. By Theorem \ref{gendim} and Theorem \ref{whitentrain1},   
    \begin{align*}
    \text{MSE}(\text{de-whiten})-\text{MSE}(\text{multi})
    &=
    \sum_{i=1}^n \lambda_i - j^{*} c - \sum_{i=j^{*}+1}^n \lambda_i
    \nonumber 
    \\
    &=\sum_{i=1}^{j^{*}} \lambda_i +  \sum_{i=j^{*}+1}^{n} \lambda_i - j^{*} c - \sum_{i=j^{*}+1}^n \lambda_i
    \nonumber \\
    &=\sum_{i=1}^{j^{*}} (\lambda_i - c)\ge 0.
    \end{align*}
    since by definition $\lambda_i\ge c$, for all $i\le j^{*}$.
\end{proof}

\section{Proof of Theorem \ref{stdtrain1}}

We reformulate Theorem \ref{stdtrain1} to also include information about the structure of the global minima, and consequently the target predictions regarding those. 

\begin{theorem} \label{standtrain}
 Let $c:=\lH \lW$. Any global minimum $(\bar{\bH}, \bar{\bW})$ of the regularized UFM-loss
\begin{equation} \label{standobj}
\frac{1}{2M} ||\bW \bH -\bY^{nrm}||_F^2 + \frac{\lH}{2M} ||\bH||_F^2+\frac{\lW}{2}||\bW||_F^2,
\end{equation}
where $\lH$ and $\lW$ are non-negative regularization parameters, takes the following form.

If $0<c<\tilde{\lambda}_{\min}$, then for any semi-orthogonal matrix $\bR$,
\[
\bar{\bW}  =  \left( {\frac{\lH}{\lW}} \right)^{1/4} \bar{\bA}^{1/2}\bR, \hspace{.2in}
\bar{\bH}  =   \sqrt{\frac{\lW}{\lH}} \bar{\bW}^{T}  \bP^{-1/2} \bY^{nrm},
\]
\[
\bar{\bW} \bar{\bH}=[\bI_n-\sqrt{c} \bP^{-1/2}] \bY^{nrm},
\]
where $\bar{\bA}=\bP^{1/2}-\sqrt{c}\bI_n$.

If $c > \tilde{\lambda}_{\max}$, then $(\bar{\bH}, \bar{\bW})=(\mathbf{0}, \mathbf{0})$. 

Furthermore, 
\[
\textnormal{MSE}(\textnormal{de-normalize}) = 
\begin{cases}
c \displaystyle \sum_{i=1}^n \lambda_i, &\text{if } 0<c < \tilde{\lambda}_{\min},
\\
\displaystyle \sum_{i=1}^n \lambda_i, &\text{if } c>\tilde{\lambda}_{\max},
\end{cases}
\]
where $\tilde{\lambda}_{\min}$ and $\tilde{\lambda}_{\max}$ are the min and max eigenvalues of the original sample correlation matrix $\bP$.
\end{theorem}

\begin{proof}[Proof of Theorem \ref{standtrain}]
Using the decomposition of $\bSigma=\bV^{1/2} \bP \bV^{1/2}$, it is readily deduced that
\[
M^{-1} \bY^{nrm} (\bY^{nrm})^T=\bV^{-\frac{1}{2}} \bSigma \bV^{-\frac{1}{2}} = \bP.
\]

\textbf{Case $0<c<\tilde{\lambda}_{\min}$:}
\\
By \citep{andriopoulos2024prevalence}[Theorem 4.1], we have that any global minimum $(\bar{\bH}, \bar{\bW)}$ for \eqref{standobj}
takes the form:
\[
\bar{\bW}  =  \left( {\frac{\lH}{\lW}} \right)^{1/4} \bar{\bA}^{1/2}\bR, \hspace{.2in}
\bar{\bH}  =   \sqrt{\frac{\lW}{\lH}} \bar{\bW}^{T}  \bP^{-1/2} \bY^{nrm},
\]
where $\bar{\bA}=\bP^{1/2}-\sqrt{c}\bI_n$. 
The optimal predictions after normalization (but before de-normalizing) are
\begin{align*}
\bar{\bW} \bar{\bH}=\bar{\bA}^{1/2} \bR \bR^T \bar{\bA}^{1/2} \bP^{-1/2} \bY^{nrm}
=\bar{\bA} \bP^{-1/2} \bY^{nrm}
&=[\bP^{1/2}-\sqrt{c}\bI_n] \bP^{-1/2} \bY^{nrm}
\\
&=[\bI_n-\sqrt{c}\bP^{-1/2}] \bY^{nrm}
\end{align*}
Our final de-normalized predictions satisfy
\begin{align*}
\check{\bY}= [\bV^{1/2}] \bar{\bW}\bar{\bH}+\bYb
&=\bV^{1/2} [\bV^{-1/2} -\sqrt{c} \bP^{-1/2} \bV^{-1/2}] (\bY-\bYb) + \bYb
\\
&=\bY -\sqrt{c} \bV^{1/2} \bP^{-1/2} \bV^{-1/2} (\bY-\bYb)
\end{align*}
Therefore,
\[
\check{\bY}-\bY = -\sqrt{c} \bV^{1/2} \bP^{-1/2} \bV^{-1/2} (\bY-\bYb),
\]
and
\begin{align*}
\frac{(\check{\bY}-\bY) (\check{\bY}-\bY)^T}{M}&=c \bV^{1/2} \bP^{-1/2} \left[\bV^{-1/2} \bSigma \bV^{-1/2}\right] \bP^{-1/2} \bV^{1/2}
\\
&=c \bV^{1/2} \bP^{-1/2} \bP \bP^{-1/2} \bV^{1/2}
\\
&=c \bV
\end{align*}
The result for $\textnormal{MSE}(\text{de-normalize})$ readily follows by taking traces in both sides.
\\
\textbf{Case $c > \tilde{\lambda}_{\max}$:}
\\
By \citep{andriopoulos2024prevalence}[Theorem 4.1], we have that the only global minimum is $(\bar{\bH}, \bar{\bW})=(\mathbf{0}, \mathbf{0})$. 
Therefore,
\[
\check{\bY}=\bYb,
\]
and 
\[
\frac{(\check{\bY}-\bY) (\check{\bY}-\bY)^T}{M}=\frac{(\bY-\bYb) (\bY-\bYb)^T}{M}=\bSigma.
\]
The result for $\textnormal{MSE}(\text{de-normalize})$ readily follows by taking traces in both sides.
\end{proof}

\section{Feature relationships across methods} \label{featrelsec}

Since the UFM was first proposed to understand feature learning, and especially the neural collapse phenomenon in classification \citep{papyan2020prevalence} and neural multivariate regression \citep{andriopoulos2024prevalence}, it can also provide insights on the feature $\bH$ learned by different methods. For example, how do the features learned by training with original targets or by whitening, and normalization, compare with each other?

Regarding this insightful question, our analysis in the Appendix (Theorem \ref{whitentrain} and the remarks before its proof, and Theorem \ref{standtrain}) explicitly discuss the nature of the optimal $\bH_*$ when whitening and normalization are applied respectively. Let us collect the globally optimal learned features and compare them across methods here as well. 

\textbf{Training with original targets} $\bY$: 
\[
\bH_* = \left(\frac{\lW}{\lH}\right)^{1/4} \bR^T [\bSigma^{1/2}-\sqrt{c} \bI_n]^{1/2} \bY^{ZCA},
\]
where $\bR\in \mathbb{R}^{n\times d}$ is semi-orthogonal and $\bY^{ZCA}=\bSigma^{-1/2}(\bY-\bYb)$.

The key observation is that the optimal learned features are formed by two procedures. The term $\bSigma^{1/2}-\sqrt{c}\bI_n$ adjusts the covariance of the original target data $\bY$. By subtracting $\sqrt{c}\bI_n$ from $\bSigma^{1/2}$, small eigenvalues are regularized or ``shrunk", effectively denoising the original target data. Taking the root of the result further scales the eigenvalues nonlinearly, emphasizing stronger signal directions. The first procedure consists of the whitened target data undergoing adaptive scaling, i.e., the multiplication $[\bSigma^{1/2}-\sqrt{c} \bI_n]^{1/2} \bY^{ZCA}$ re-weights the whitened target data using the thresholded eigenvalues from $\bSigma$. Directions aligned with strong original covariance are amplified, while weak/noisy directions are zeroed out. The second procedure consists of a rotation of the first procedure's outcome.The semi-orthogonal matrix $\bR^T$ acts as a rotation, redistributing the features into a coordinate system that preserves distances but may optimize for properties like orthogonality or sparsity. To summarize, the two procedures ensure a denoised, lower dimensional feature representation that retains only statistically significant components from the covariance of the original target data. 

\text{For single-task} $i$: 
\[
\bH_*^{(i)} = \left(\frac{\lW}{\lH}\right)^{1/4} \bR^T [\sigma_i-\sqrt{c}]^{1/2} \sigma_i^{-1} \by^{(i)},
\]
where $\by^{(i)}$ is the $i$-th row of $\bY$. There is no a priori relationship between $\bH_*$ and $\bH_*^{(i)}$. However, we note that in the special case when the targets are uncorrelated, it is easy to see that $\bH_*^{(i)}$ is the $i$-th row of $\bH_*$.

\textbf{Training with whitened targets} $\bY^{ZCA}$:
\[
\bH_* = \left(\frac{\lW}{\lH}\right)^{1/4} \bR^T [\bI_n^{1/2}-\sqrt{c} \bI_n]^{1/2} \bY^{ZCA}.
\]
In this case, the form of the optimal learned features is retained with the difference that the $n\times n$ identity matrix $\bI_n$ takes the place of $\bSigma$. The whitened target data is simply scaled down by $\sqrt{1-\sqrt{c}}$, if $c<1$, akin to mild regularization. No denoising (eigenvalue thresholding) takes place, and all the directions are kept.

\textbf{Training with normalized targets} $\bY^{nrm}=\bV^{-1/2} (\bY-\bYb)$:

Recall that we have used $\bP$ to denote the correlation matrix of the targets. Then,
\[
\bH_* = \left(\frac{\lW}{\lH}\right)^{1/4} \bR^T [\bP^{1/2}-\sqrt{c} \bI_n]^{1/2} \bY^{ZCA-cor},
\]
where $\bY^{ZCA-cor}:=\bP^{-1/2} \bY^{nrm}$ is referred to in the literature as the ZCA-cor whitening \citep{kessy2018optimal}, and it is the unique whitening procedure that makes the transformed normalized targets as similar as possible to the original normalized targets, the similarity being in terms of the cross-correlation between the former and the latter. 

The optimal leaned features are obtained in accordance to the analysis that we have outlined when training with original targets. Here, the role of $\bSigma$ is played by $\bP$ and in the place of $\bY^{ZCA}$, we have $\bY^{ZCA-cor}$.


\section{Generalization bounds} \label{genboundappendix}

Let $\mathcal{X}$ denote the input space and $\mathcal{Y}$ the target space, which regarding the learning problem of neural multivariate regression, is a subset of $\mathbb{R}^n$. Here, we adopt the stochastic scenario and will denote by $\mathcal{D}$ a distribution over $\mathcal{X}\times \mathcal{Y}$. In the supervised learning scenario, the learner receives training examples $S:=\{(\bx_i, \by_i),i=1,...,M\}\in (\mathcal{X}\times \mathcal{Y})^M$ drawn in a i.i.d. manner according to $\mathcal{D}$. The deterministic scenario where input points admit a unique target value determined by a target function $f: \mathcal{X}\to \mathcal{Y}$ is a straightforward special case.

We denote by $L: \mathcal{Y}\times \mathcal{Y}\to \mathbb{R}_{+}$ the loss function used to measure the magnitude of the difference between the vector-valued target predicted and the ``true" or ``correct" one. The most common loss function used in neural multivariate regression is the $L_p$-loss defined by $L(\by, \by')=||\by-\by'||_p$, for some $p\ge 1$, and every $\by, \by'\in \mathcal{Y}$.

Given a hypothesis set $\mathcal{H}$, that is a set which contains all the functions mapping the input space $\mathcal{X}$ to the target space $\mathcal{Y}$, neural multivariate regression tasks consist of using a set of training examples $S$ to find a hypothesis $h\in \mathcal{H}$ with small generalization error $R(h)$ with respect to the ``true" or ``correct" \textit{target function} mapping inputs to targets in $S$.

\begin{definition} [Generalization error]
    Given a hypothesis $h\in \mathcal{H}$, a \textit{target function}, and an underlying distribution $\mathcal{D}$, the generalization error is defined by
    \[
    R(h):= \mathbb{E}_{(\bx, \by)\sim \mathcal{D}}\left[L(h(\bx), \by)\right],
    \]
    where $L: \mathcal{Y}\times \mathcal{Y}\to \mathbb{R}_{+}$ is the loss function used to measure the magnitude of error.
\end{definition}

The generalization error is not directly accessible to the learner since both the distribution $\mathcal{D}$ and the \textit{target function} are unknown. However, the learner can measure the empirical error of a hypothesis on a set of training examples $S$.

\begin{definition} [Training error]
  Given a hypothesis $h\in \mathcal{H}$, a \textit{target function}, and a training set $S$, the training error is 
  \[
  \hat{R}_S(h):= \frac{1}{M} \sum_{i=1}^M L(h(\bx_i), \by_i).
  \]
\end{definition}

Thus, the training error of $h$ is its average error over the training set $S$, while the generalization error is its expected error based on the distribution $\mathcal{D}$. If $L$ is the squared loss, the training error represents the training MSE of $h$ on $S$.  

The theoretical results presented below are based on the assumption that the neural multivariate regression problem is bounded, that is when the loss function is bounded above by some $K>0$, i.e., $L(\by, \by')\le K$, for all $\by, \by'\in \mathcal{Y}$, or, more strictly, when $L(h(\bx), \by)\le K$, for all $h\in \mathcal{H}$, and $(\bx, \by)\in \mathcal{X}\times \mathcal{Y}$.

\subsection{Generalization bounds for kernel-based hypotheses}

Generalization bounds based on Rademacher complexity (RC), a term that captures the richness of a family of functions by measuring the degree to which a hypothesis set can fit random noise, were presented in \citep{mohri2018foundations}, e.g., Theorem 11.3 therein. These generalization bounds suggest a trade-off between reducing the training MSE and controlling the RC of the hypothesis set. A richer or more complex hypothesis set achieves a small training MSE but has high RC, while a poorer or more simple hypothesis set has small RC but achieves high training MSE. The aim is to control this trade-off. An important benefit of the learning bounds in \cite{mohri2018foundations}[Theorem 11.3] is that they are data dependent. This can lead to more accurate learning guarantees. For kernel-based hypotheses upper bounds on the RC can be used directly to derive generalization bounds depending on the trace of the kernel matrix or the maximum diagonal entry.  

\begin{definition}[Kernel]
    A function $K:\mathcal{X}\times \mathcal{X}\to \mathbb{R}$ is said to be a positive definite symmetric (PDS) kernel if for any $\{\bx_i:i=1,...,M\}\in \mathcal{X}$, the matrix $K:=[K(\bx_i, \bx_j)]_{i, j}\in \mathbb{R}^{M\times M}$ is symmetric positive semi-definite (SPSD).
\end{definition}

For a PDS kernel $K:\mathcal{X}\times \mathcal{X}\to \mathbb{R}$, there exists a Hilbert space $\mathbb{H}$ and a mapping $\Phi: \mathcal{X}\to \mathbb{H}$ such that:
\[
K(\bx, \bx')=\left<\Phi(\bx), \Phi(\bx')\right>, \qquad \forall \bx, \bx'\in \mathcal{X}.
\]
For a proof of this result, we refer the reader to \cite{mohri2018foundations}[Theorem 6.8]. 

A generalization bound for (univariate) linear regression with bounded linear hypotheses in a feature space defined by a PDS kernel was presented in \cite{mohri2018foundations}[Theorem 11.11]. For simplicity, we give the generalization bound for the squared loss.  

\begin{theorem}
    Let $K:\mathcal{X}\times \mathcal{X}\to \mathbb{R}$ be a PDS kernel, $\Phi: \mathcal{X} \to \mathbb{H}$ be a feature mapping associated with $K$, and 
    \[
    \mathcal{H}:=\{\bx\mapsto \bw\cdot \Phi(\bx): ||\bw||\le \Lambda_{\bw}\}
    \]
    be the family of bounded linear hypotheses corresponding to the optimization problem
    \[
    \min_{\bw} \frac{1}{M} \sum_{i=1}^{M} (\bw\cdot \Phi(\bx_i)-\by_i)^2, \textnormal{ subject to } ||\bw||\le \Lambda_{\bw},
    \]
    for a training set $S=\{(\bx_i, \by_i): i=1,...,M\}$.
    \begin{itemize}
        \item Assume that there exists $K>0$ such that $|h(\bx)-\by|\le K$, for all $(\bx, \by)\in \mathcal{X}\times \mathcal{Y}$.
        \item Let $\text{tr}\left([K(\bx_i, \bx_j)]_{i, j}\right)\le M r^2$, for any training set $S$ of size $M$.
    \end{itemize}
    Then, for any $\delta>0$, with probability at least $1-\delta$, the following inequality holds for all $h\in \mathcal{H}$:
    \begin{equation} \label{firstgenbound}
        \mathbb{E}_{(\bx, \by)\in \mathcal{D}}[|h(\bx)-\by|^2]\le \textnormal{MSE}_S(h)+ 4 K \frac{r \Lambda_{\bw}}{\sqrt{M}} + 3 K^2 \sqrt{\frac{\log \frac{2}{\delta}}{2 M}}.
    \end{equation}
\end{theorem}

The generalization bound of the theorem above suggest minimizing a trade-off between the training MSE, denoted in \eqref{firstgenbound} by $\text{MSE}_S(h)$, and the norm of the weight vector $\bw$. The third term adds an error dependent on the confidence level $\delta$ and the size of the training set $M$.

\subsection{Application to the UFM via the Layer-Peeled Model}

Another formulation of the UFM is the so-called Layer-Peeled Model introduced by \cite{fang2021exploring} as:
\begin{equation} \label{layerpeeled}
    \min_{\bW, \bH} ||\bW \bH - \bY||_F^2, \text{ subject to }  
\begin{cases}
    ||\bW||_F&\le \Lambda_{\bW},
    \\ 
    \\
    \displaystyle \frac{||\bH||_F}{\sqrt{M}}&\le \Lambda_{\bH},
\end{cases}
\end{equation}
where $\bW\in \mathbb{R}^{n\times d}$ is, as in \eqref{formofloss}, a linear classifier in the last layer, $\bH=[\bh_1,...,\bh_M]\in \mathbb{R}^{d\times M}$, where $\bh_i:=\bh_{\theta}(\bx_i)$ is the $d$-dimensional last-layer activation/feature of the $i$-th training sample, and $\Lambda_{\bW}, \Lambda_{\bH}$ are positive scalars constraining the matrix-norms of $\bW$ and $\bH$ respectively.

A constrained minimization problem can be solved by means of the Karush-Kuhn-Tucker (KKT) multiplier method, which minimizes a function subject to inequality constraints. The KKT multiplier method states that, under some regularity conditions (all met here), there exist constants $\nu_{\bW}, \nu_{\bH}\ge 0$, called the multipliers, such that the solution $(\bW(\nu_{\bW}), \bH(\nu_{\bH}))$ of the constrained minimization problem \eqref{layerpeeled} satisfies the so-called KKT conditions.

\begin{itemize}
\item The first condition (referred to as the stationarity condition) demands that the gradients of the Langrangian 
\[
\Delta(\bW, \bH):= ||\bW \bH - \bY||_F^2 + \nu_{\bW} (||\bW||_F^2-\Lambda_{\bW}^2) + \nu_{\bH} (||\bH||_F^2-M \Lambda_{\bH}^2),
\]
associated with the minimization problem \eqref{formofloss}, i.e., the UFM, is 0 at the solution $(\bW(\nu_{\bW}), \bH(\nu_{\bH}))$. More specifically, 

 \begin{align} \label{kkt1}
    \frac{\partial \Delta}{\partial \bW}\bigg |_{(\bW(\nu_{\bW}), \bH(\nu_{\bH}))}&= 2 (\bW(\nu_{\bW}) \bH(\nu_{\bH}) - \bY) \bH(\nu_{\bH})^T+ 2\nu_{\bW} \bW(\nu_{\bW}),
    \\
    \frac{\partial \Delta}{\partial \bH}\bigg |_{(\bW(\nu_{\bW}), \bH(\nu_{\bH}))}&= 2 \bW(\nu_{\bW})^T ( \bW(\nu_{\bW}) \bH(\nu_{\bH})  - \bY) + 2\nu_{\bH} \bH(\nu_{\bH}).
\end{align}

\item The second KKT condition (referred to as the complimentarity condition) requires that 
\begin{equation} \label{kkt2}
\nu_{\bW} (||\bW(\nu_{\bW})||_F^2-\Lambda_{\bW}^2)=0, \qquad \nu_{\bH}(||\bH(\nu_{\bH})||_F^2-M \Lambda_{\bH}^2)=0.
\end{equation}
\end{itemize}
If 
\begin{align*}
    \nu_{\bW}=\lW, \qquad &\nu_{\bH}=\lH , 
    \\
    \Lambda_{\bW}=||\bW(\lW)||_F, \qquad &\Lambda_{\bH}=\frac{||\bH(\lH)||_F}{\sqrt{M}},
    \end{align*}
the UFM solution $(\bW(\lW), \bH(\lH))$ satisfies \eqref{kkt1}-\eqref{kkt2}. Therefore, the theorem that follows is immediately deduced. 

\begin{theorem}
    If 
    \[
    \Lambda_{\bW} = ||\bW(\lW)||_F, \qquad \Lambda_{\bH} = \frac{||\bH(\lH)||_F}{\sqrt{M}},
    \]
    the minimization problems of the UFM \eqref{formofloss} and the Layer-Peeled Model \eqref{layerpeeled} have the same solution.
\end{theorem}

\begin{remark}
    The UFM solutions $\bW(\lW)$ and $\bH(\lH)$ are always to be found on the boundary of the Layer-Peeled Model constraints, parameterized by $\{(\bW, \bH): ||\bW||_F\le \Lambda_{\bW}, ||\bH||_F\le M \Lambda_{\bH}\}$ for some $\Lambda_{\bW}, \Lambda_{\bH}>0$. The size of the spherical constraints of the Layer-Peeled Model shrink as the regularizing constants of the UFM increase, and eventually in the $\lW, \lH\to \infty$-limit, $\Lambda_{\bW}, \Lambda_{\bH}\to 0$. This follows from the closed-form functions of the minimizers for the UFM with respect to each other, i.e.,  
    \begin{align*}
    &\lim_{\lw\to \infty} \bW(\lW)=\lim_{\lW\to \infty} \bY \bH^T [\bH \bH^T+M \lW]^{-1}=0, 
    \\
    &\lim_{\lH\to \infty} \bH(\lH) = \lim_{\lH\to \infty} \bW^T [\bW \bW^T+\lH \bI_n]^{-1} \bY = 0.
    \end{align*}
\end{remark}

\begin{remark}
    Suppose $c<\lambda_{\min}$. By applying \cite{andriopoulos2024prevalence}[Corollary 4.2 (ii)-(iii)] to the $n$-dimensional case, the following relation holds:
    \begin{align*}
    \Lambda_{\bW}&=||\bW(\lW)||_F=\left(\frac{\lH}{\lW}\right)^{1/4} \left[\text{tr}\left(\bSigma^{1/2}-\sqrt{c}\bI_n\right)\right]^{1/2},
    \\
    \Lambda_{\bH}&=\frac{||\bH(\lH)||_F}{\sqrt{M}}
    =\sqrt{\frac{\lW}{\lH}} ||\bW(\lW)||_F = \sqrt{\frac{\lW}{\lH}} \Lambda_{\bW}.
    \end{align*}
\end{remark}

Recall that $\bh_i:=\bh_{\theta}(x_i)\in \mathbb{R}^d$ for all $i=1,...,M$, where $\bh_{\theta}$ is associated with the PDS kernel $K=\bH^T \bH$, i.e., $[\bh_i\cdot \bh_j]_{i, j}\in \mathbb{R}^{M\times M}$ is the covariance matrix of the $\bh_i$'s and as such it is symmetric and positive semi-definite. Under the constraints that $\bW$ and $\bH$ are subject to:
\[
\text{tr}(K)=\text{tr}(\bH^T \bH)=||\bH||_F^2\le M \Lambda_{\bH}^2.
\]
Under the UFM, when $c<\lambda_{\min}$, $\text{MSE}_S(h)=c=\lW \lH$ for any labeled sample $S$, see Theorem \ref{gendim}. Combining the points above, the generalization bound of \eqref{firstgenbound} yields
\[
    \mathbb{E}_{(\bx, \by)\in \mathcal{D}}[|h(\bx)-\by|^2]\le n c +\mathcal{O}\left(\frac{C}{\sqrt{M}}\right) + \mathcal{O}\left(\sqrt{\frac{\log 2 \delta^{-1}}{M}}\right),
\]
where $C:=\Lambda_{\bW} \Lambda_{\bH}$. In the special case in which we set $\lW=\lH$, we have $c=\lW^2$, and $C=\Lambda_{\bW}^2=\text{tr}\left(\bSigma^{1/2}-\lW \bI_n\right)$, and the generalization bound reads
\begin{equation} \label{2ndgenbound}
    \mathbb{E}_{(\bx, \by)\in \mathcal{D}}[|h(\bx)-\by|^2]\le n \lW^2 +\mathcal{O}\left( \frac{\text{tr}\left(\bSigma^{1/2}-\lW \bI_n\right)}{\sqrt{M}}\right) + R, 
\end{equation}
where, for a fixed confidence level $\delta\in (0,1)$, $\lim_{M\to \infty} R=0$.

For single task $i$, the right-hand side (wit the remainder term) of \eqref{2ndgenbound} becomes 
\[
\text{GB}^{(i)}:=\lW^2 + \mathcal{O} (M^{-1/2}(\sigma-\lW)) + R.
\]
Using this, we can directly derive an upper bound for the generalization error of the $n$-single tasks neural regression problem:
\begin{equation} \label{3rdgenbound}
\text{GB}(\text{n-single)}\le \sum_{i=1}^n \text{GB}^{(i)}= n \lW^2 + \mathcal{O}\left(\frac{\sum_{i=1}^n\sigma_i-n \lW}{\sqrt{M}}\right)+\tilde{R},
\end{equation}
where, for a fixed confidence level $\delta\in (0,1)$, $\lim_{M\to \infty} \tilde{R}=0$.

Because $\sum_{i=1}^n \sigma_i=\sum_{i=1}^n \sqrt{\Sigma_{ii}}\ge \text{tr}(\bSigma^{1/2})$, the $\mathcal{O}$-term in \eqref{3rdgenbound} is $\ge $ than the $\mathcal{O}$-term in \eqref{2ndgenbound}, so the whole right-hand side of the bound in \eqref{3rdgenbound} is $\ge $ than the right-hand side of the bound in \eqref{3rdgenbound}.

If the two test MSEs concentrate near their respective bounds, then the multi-task test MSE is smaller than that of the $n$ single tasks; thus the gain from multi-tasking is tightest when $\bSigma$ is non-diagonal and the features across tasks are correlated.

\end{document}